\documentclass{article}

\usepackage{PRIMEarxiv}
\usepackage{xspace}
\usepackage[utf8]{inputenc} % allow utf-8 input
\usepackage[T1]{fontenc}    % use 8-bit T1 fonts
\usepackage{hyperref}       % hyperlinks
\usepackage{url}            % simple URL typesetting
\usepackage{booktabs}       % professional-quality tables
\usepackage{amsfonts}       % blackboard math symbols
\usepackage{amsmath, amssymb}
\usepackage{nicefrac}       % compact symbols for 1/2, etc.
\usepackage{microtype}      % microtypography
\usepackage{lipsum}
\usepackage{fancyhdr}       % header
\usepackage{graphicx}       % graphics
\usepackage[table]{xcolor}
\definecolor{customred}{RGB}{220,20,60}
\usepackage[most]{tcolorbox}
\usepackage{subcaption}
\graphicspath{{media/}}     % organize your images and other figures under media/ folder
\usepackage{enumitem}
\usepackage{appendix}
\usepackage{bm}
\usepackage{amsmath}
\usepackage{mathtools}
\usepackage{amssymb}
\usepackage{amsthm}
\usepackage{comment}
\usepackage{natbib}
\usepackage{booktabs}
\usepackage{graphicx}
\usepackage{algorithm}
\usepackage{float}
\usepackage{multirow}
\usepackage{dsfont}
\usepackage{tcolorbox}
\usepackage{tabulary}
\usepackage{colortbl}
\usepackage{color}
\usepackage{hyperref}
\usepackage{url}
\usepackage{enumitem}
\usepackage[utf8]{inputenc} % allow utf-8 input
\usepackage[T1]{fontenc}    % use 8-bit T1 fonts
\usepackage{hyperref}       % hyperlinks
\usepackage{url}            % simple URL typesetting
\usepackage{booktabs}       % professional-quality tables
\usepackage{amsfonts}       % blackboard math symbols
\usepackage{nicefrac}       % compact symbols for 1/2, etc.
\usepackage{microtype}      % microtypography
\usepackage{xcolor}         % colors
\usepackage{enumitem}
\usepackage{amsmath, amssymb} % Add necessary packages
\usepackage{bm} % For bold math symbols
\usepackage{graphicx} % \resizebox
\usepackage{multirow}
\usepackage{colortbl}
\usepackage{hhline}
\usepackage{dashrule}
\usepackage{caption}
\usepackage{algorithm}

\usepackage{geometry}
\usepackage{algpseudocode}
\usepackage{graphicx} % For \resizebox
\usepackage{booktabs} % For \toprule, \midrule, \bottomrule
\usepackage{array}    % For column specifications
\usepackage{xcolor}   % For \arrayrulecolor and \textcolor (though not used in this specific example for text)
\usepackage{caption}  % For \captionof

\usepackage{wrapfig}
\usepackage{amsthm} 
\usepackage{fancyhdr}
\definecolor{LightCyan}{rgb}{0.88,1,1}
\definecolor{LightRed}{rgb}{1,0.5,0.5}
\definecolor{LightYellow}{rgb}{1,1,0.88}
\definecolor{Grey}{rgb}{0.75,0.75,0.75}
\definecolor{DarkGrey}{rgb}{0.55,0.55,0.55}
\definecolor{LightGreen}{rgb}{0.0, 0.70, 0.0}

\definecolor{gen}{RGB}{0,0,200}
\definecolor{cc}{RGB}{231,117,0}

\allowdisplaybreaks

\newcommand{\defn}{\coloneqq}

\newcommand{\train}{\mathsf{train}}
\title{Beyond Confidence: Adaptive and Coherent Decoding for Diffusion Language Models}

\author{
  Kecheng Chen$^{1,\bigstar}$, Ziru Liu$^{2,\bigstar,\spadesuit}$, Xijia Tao$^{3}$, 
  Hui Liu$^{1}$, Xinyu Fu$^{2}$, Suiyun Zhang$^{2}$, \\ \textbf{Dandan Tu}$^{2}$,  
  \textbf{Lingpeng Kong}$^{3}$, \textbf{Rui Liu}$^{2,\clubsuit}$, \textbf{Haoliang Li}$^{1,\clubsuit}$ \\
  \\
  $^{1}$City University of Hong Kong, 
  $^{2}$Huawei Research, 
  $^{3}$The University of Hong Kong 
  \\
  Email: \texttt{liu.rui2@huawei.com}; \texttt{ haoliang.li@cityu.edu.hk}\\
}

\begin{document}
\theoremstyle{plain} 
\newtheorem{lemma}{\bf Lemma} 
\newtheorem{proposition}{\bf Proposition}
\newtheorem{theorem}{\bf Theorem}
\newtheorem{corollary}{\bf Corollary} 
\newtheorem{claim}{\bf Claim}

\theoremstyle{remark}
\newtheorem{assumption}{\bf Assumption} 
\newtheorem{definition}{\bf Definition} 
\newtheorem{condition}{\bf Condition}
\newtheorem{property}{\bf Property} 
\newtheorem{example}{\bf Example}
\newtheorem{fact}{\bf Fact}
\newtheorem{remark}{\bf Remark}

\maketitle

\begin{abstract}
Diffusion Language Models (DLMs) have recently achieved significant success due to their any-order generation capabilities. However, existing inference methods typically rely on local, immediate-step metrics—such as confidence or entropy—which inherently lack a more reliable perspective. This limitation frequently leads to inconsistent sampling trajectories and sub-optimal generation quality. To address this, we propose \textbf{C}oherent \textbf{C}ontextual \textbf{D}ecoding (\textbf{CCD}), a novel inference framework built upon two core innovations. First, CCD employs a trajectory rectification mechanism that leverages historical context to enhance sequence coherence, enabling the early rejection of sub-optimal paths. We demonstrate that this mechanism is theoretically equivalent to modeling the consistency of historical steps via the conditional mutual information between context and token predictions. Building on this theoretical insight, we further address the inefficiency of conventional uniform decoding budgets. Instead of rigid allocations based on diffusion steps, we introduce an adaptive sampling strategy that dynamically adjusts the unmasking budget for each step according to our consistency metric. Consequently, our method significantly improves the quality of generation trajectories while accelerating the sampling process. Empirically, our method achieves a simultaneous enhancement in both inference speed and performance across diverse benchmarks on Dream and LLaDA, delivering up to $3.48\times$ speedup alongside 3.91\% performance improvement. %The implementation code is available online to ease reproducibility.
\end{abstract}

\footnotetext[1]{$^{\bigstar}$ Contribute equally to this work $^{\clubsuit}$Corresponding authors $^{\spadesuit}$ Project Lead}
% keywords can be removed

\section{Introduction}
\label{sec:Introduction1}

Recently, modern Large Language Models (LLMs) have demonstrated exceptional capabilities across diverse tasks, including coding, multi-turn dialogue, and mathematical reasoning \citep{kimiteam2025kimik2openagentic,openAI-o3,deepseekai2025deepseekv3technicalreport}. The prevailing paradigm for these models relies on autoregressive transformer architectures \citep{vaswani2023attentionneed}, which utilize a causal mask to perform next-token prediction via left-to-right generation. To transcend the constraints of this paradigm, the community is actively exploring next-generation architectures capable of mitigating the inherent inductive biases of autoregressive models \citep{austin2021structured,lou2023discrete,shi2024simplified,sahoo2024simple,gong2024scaling}. Notably, Diffusion Language Models (DLMs) have emerged as a compelling alternative \citep{gemini-diffusion, inceptionlabs2025mercury}. Unlike autoregressive models restricted to unidirectional contexts, DLMs leverage bidirectional contexts, incorporating information from all positions simultaneously. This global perspective allows DLMs to potentially achieve superior global awareness and long-term planning capabilities. With recent scaling to the 7--8 billion parameter regime (e.g., the Dream \citep{ye2025dream} and LLaDA \citep{nie2025large}), masked diffusion mechanisms have proven capable of matching the performance of their autoregressive counterparts while exhibiting significantly better long-term planning \citep{ye2025dream}. However, in contrast to the rapid advancements in training strategies, the inference (or sampling) procedure of DLMs remains relatively underexplored. Specifically, DLMs iteratively decode semantically meaningful tokens from input mask tokens based on the previously decoded context. Existing methods typically leverage the predictive distribution at the current step to select the top-$k$ or most certain tokens for unmasking \citep{ye2025dream}. To balance performance and latency, the decoding budget is usually kept uniform across all iterations \citep{ye2025dream, nie2025large}. While some approaches employ dynamic strategies based on confidence levels \citep{wei2025acceleratingdiffusionlargelanguage, li2025diffusion, hong2025wide} or speculative techniques \citep{li2025diffuspecunlockingdiffusionlanguage}, the standard paradigm remains largely heuristic.

Despite their empirical success, we argue that current single-step sampling procedures with uniform decoding budgets suffer from three critical limitations: (i) \textbf{Susceptibility to Local Optima:} Whether relying on entropy or maximum probability, single-step certainty-based schemes are prone to settling for local optima. This issue is exacerbated by the tendency of modern LLMs to exhibit over-confident even in error predictions~\citep{sun2025largelanguagemodelsoverconfident}. (ii) \textbf{Lack of Theoretical Grounding:} Current sampling procedures lack a direct theoretical bridge to the sampling error rate, making it difficult to govern the decoding process for controllable performance guarantees. (iii) \textbf{Suboptimal Budget Allocation:} Intuitively, a uniform sampling scheme is inefficient \citep{yang2025tamingmaskeddiffusionlanguage} on some sceneries. For instance, context-insensitive generations (e.g., inherent templates or formulaic phrases) require a smaller decoding budget, whereas complex reasoning steps require a larger one. While recent works have attempted to improve reliability—such as temporal self-consistency voting~\citep{wang2025time} or End-of-Sequence (EOS) token rejection with ascending budgets~\citep{yang2025tamingmaskeddiffusionlanguage}—these solutions often incur significant computational overheads due to complex post-processing or rely on handcrafted heuristics that lack a solid theoretical foundation.

To address these challenges collectively, we introduce \textbf{C}oherent \textbf{C}ontextual \textbf{D}ecoding (\textbf{CCD}), a novel inference framework that redefines DLM inference as a contextual consistency-aware, multi-step decoding process. CCD adopts a new perspective by modeling the consistency of historical steps using the conditional mutual information between the context and token predictions. This provides a theoretically grounded, fine-grained measure of consistency. Furthermore, we address the inefficiency of rigid decoding schedules by introducing an adaptive sampling strategy that dynamically adjusts the unmasking budget based on our consistency metric. Our key contributions are summarized as follows:
\begin{figure}[t]
  \centering
  \includegraphics[width=0.95\linewidth]{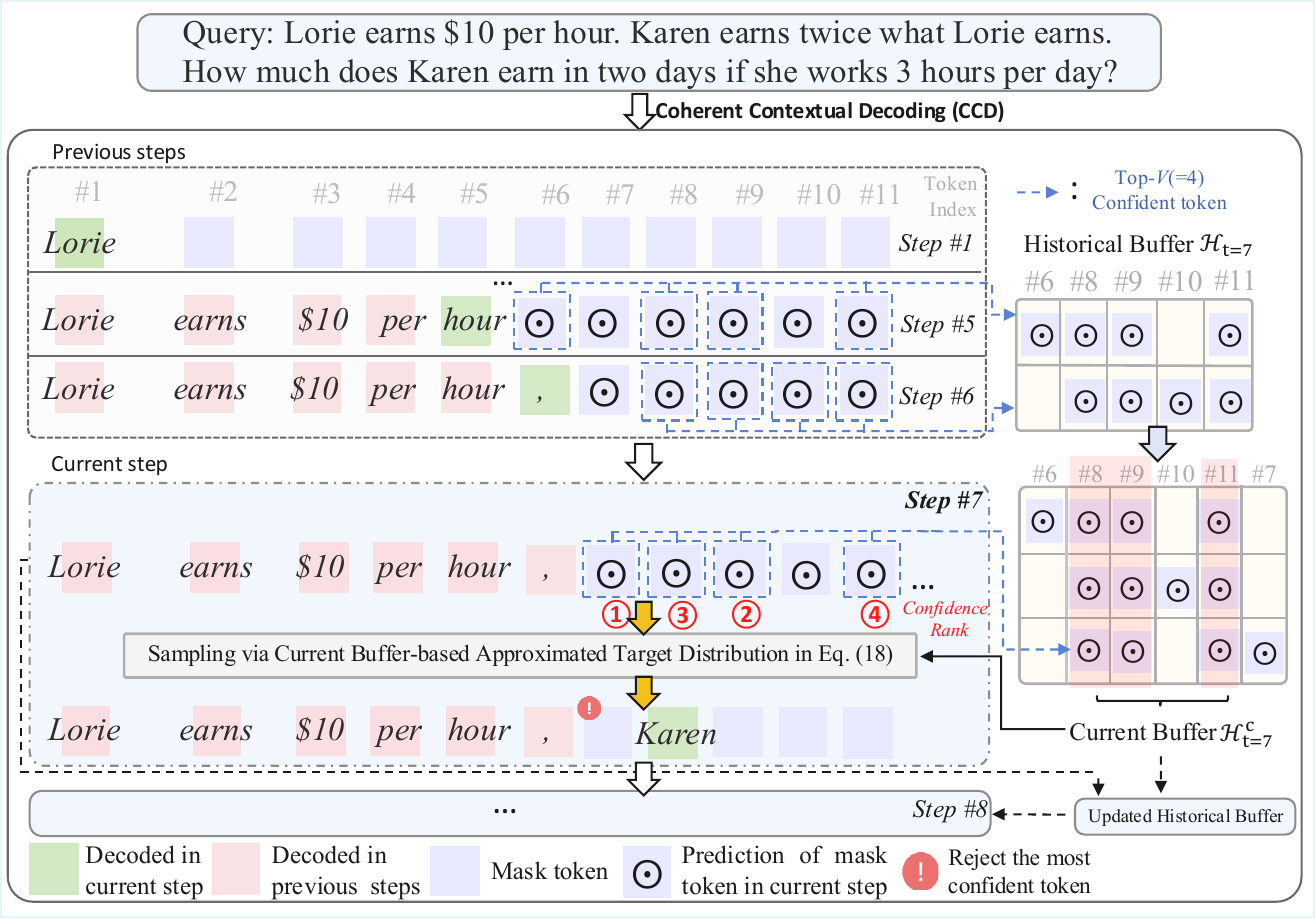}
  \caption{Framework of our proposed method.  We define a historical buffer $\mathcal{H}_t$ at iteration $t$ that stores the predictive distributions from the most recent $d$ iterations (except for the current iteration $t$) with only the top-$V$ most confident tokens at each iteration. At the current iteration $t$, we also identify mask token positions that appear both in the current top-$V$ set and in the historical buffer to obtain the current buffer $\mathcal{H}_{t}^{c}$, which ensures consistently confident tokens with maximized effective contexts to conduct an approximated target distribution-based sampling procedure in Eq. (\ref{pratical decoding}).  }

  \label{fig:diffuspec-framework}
  \vspace{-0.5cm}
\end{figure}

\begin{itemize}[leftmargin=*]
\item We propose Coherent Contextual Decoding, a training-free decoding framework that employs a \textit{trajectory rectification mechanism}. By leveraging historical context to enhance sequence coherence, CCD enables the early rejection of unreliable current-step optima.
\item We provide a rigorous theoretical analysis demonstrating that our mechanism is equivalent to modeling the consistency of historical steps via conditional mutual information. This formulation implicitly constructs a bridge to the theoretical upper bound of the sampling error rate.
\item We further introduce an adaptive sampling strategy that replaces rigid diffusion steps with a dynamic budget allocation guided by our consistency metric. This significantly improves the efficency of the sampling process.
\item Empirically, we demonstrate our method achieve simultaneous enhancement in both inference speed and performance across diverse benchmarks on Dream and LLaDA, delivering up to $3.48\times$ speedup alongside 3.91\% performance improvement.
\end{itemize}
\section{Preliminaries}
\label{sec:Preliminaries}

\textbf{Notation.} Let $\overline{\mathbb{X}} = \mathbb{X} \cup \{\operatorname{M}\}$ be the extended vocabulary of a finite-size vocabulary $\mathbb{X}$, where $\operatorname{M}$ denotes an absorbing mask token. Let $\mathbf{s} \in \mathbb{X}^{L}$ be a prompt sequence of length $L$. Let $\mathbf{x}_{t} = \{x_{t,1}, \dots, x_{t,N}\} \in \overline{\mathbb{X}}^{N}$ denote the response sequence at time step $t$, where $t \in \{0, \dots, T\}$ indicates the diffusion step. We adhere to the standard diffusion formulation where the process evolves from a fully masked state $\mathbf{x}_{T} = \{\operatorname{M}\}^{N}$ to a fully unmasked data state $\mathbf{x}_{0} \in \mathbb{X}^{N}$. We denote the concatenation of the prompt and response at step $t$ as $[\mathbf{s}; \mathbf{x}_{t}]$. Hereafter, $\hat{\mathbf{x}}_{t}$ and $\overline{\mathbf{x}}_{t}$ denote the predicted and true response sequences at iteration $t$, respectively.

\begin{figure}[h]
    \centering
    \begin{subfigure}{0.49\textwidth}
        \centering
        \includegraphics[width=\textwidth]{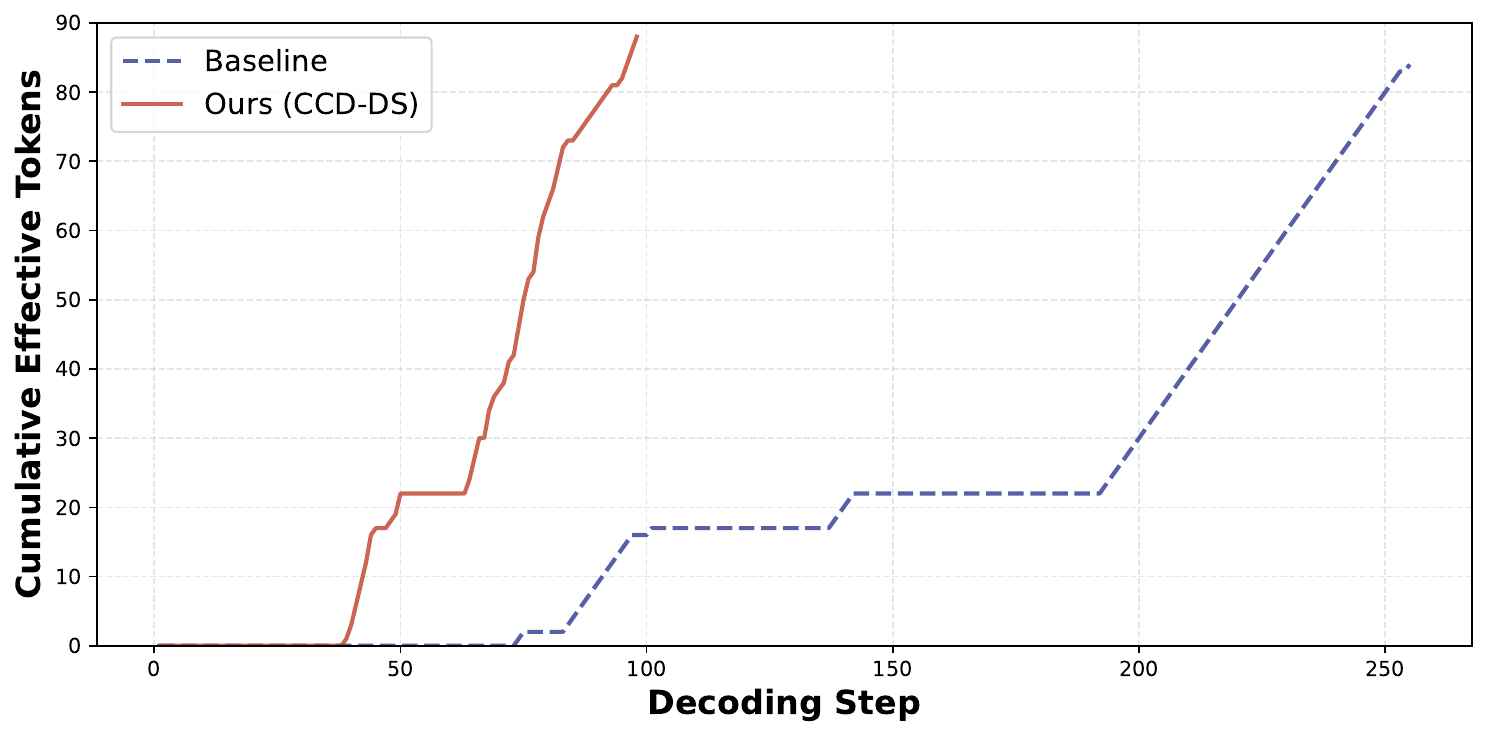}
        \caption{}
        \label{fig:sub1}
    \end{subfigure}
    \hfill
    \begin{subfigure}{0.49\textwidth}
        \centering
        \includegraphics[width=\textwidth]{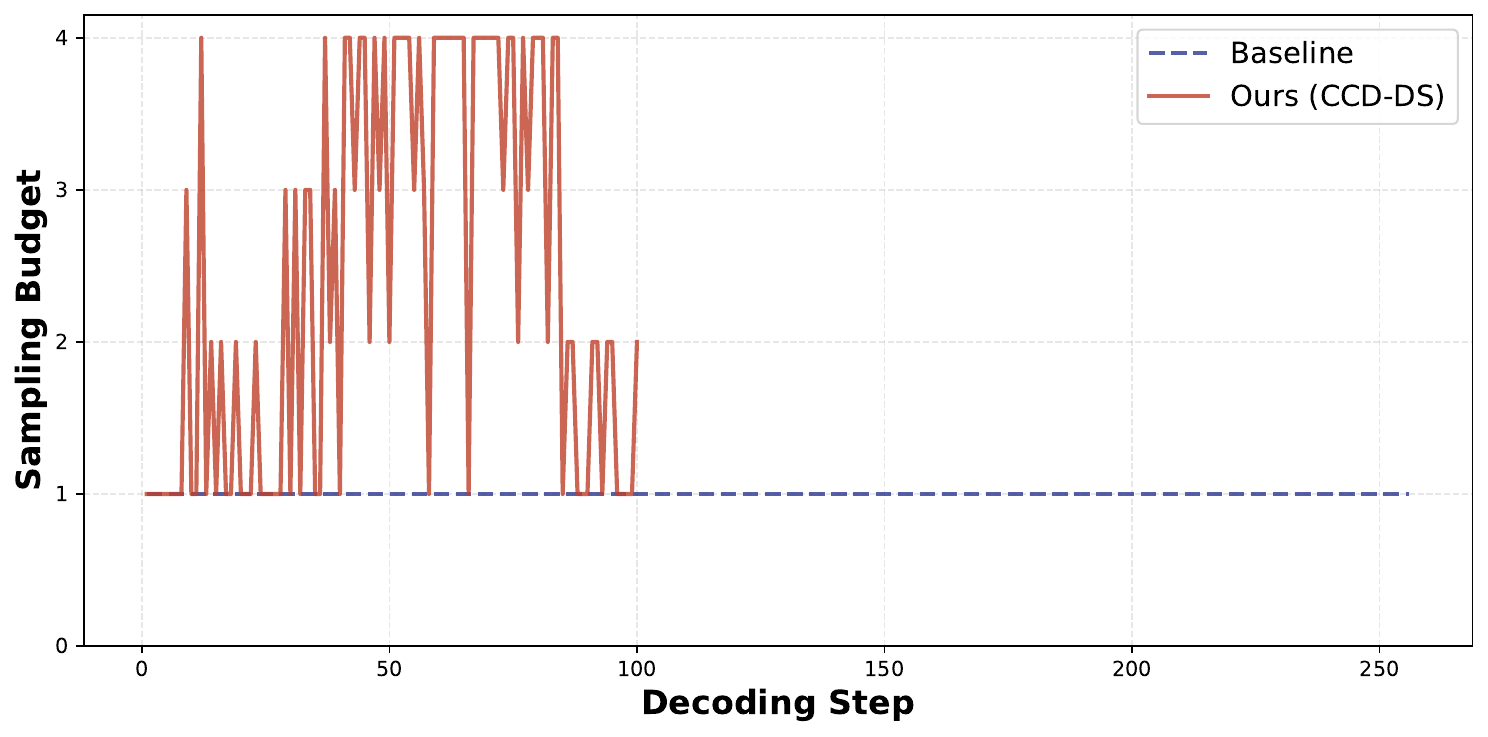}
        \caption{}
        \label{fig:sub2}
    \end{subfigure}
    \caption{This analysis investigates the Cumulative effective tokens (CET) and sampling budget using the Dream model on the Trip benchmark. The sequence length is set to 256. (a) We calculate the CET by excluding padding/EOS tokens from the total generated tokens. The decoding process includes periods where multiple EOS tokens are generated (visible as \textbf{plateaus}). (b) Sampling budget of each diffusion step of different sampling procedures. For our method, the decoding process is early stopped because all mask tokens are decoded.}
    \label{fig:adaptive sampling}
\end{figure}

\textbf{Sampling procedure.} A Diffusion Language Model (DLM) approximates the reverse process as a sequence-to-sequence mapping parameterized by $\theta$, which outputs a predictive distribution $\hat{p}_\theta$. The model progressively decodes mask tokens starting from $\mathbf{x}_{T}$. At each iteration, the model takes the fixed prompt $\mathbf{s}$ and the state from the previous iteration $\mathbf{x}_{t+1}$ to estimate the clean data $\mathbf{x}_0$.

Typically, the generation budget $b_{t}$ (the number of tokens to unmask at step $t$) is uniform across iterations (as illustrated in the dashed curve of Figure \hyperlink{fig:adaptive sampling}{2(b)}), such that $b_{t} \approx N/T$ for all $t \in [T]$. Under this setup, the most widely adopted sampling schedule\footnote{Random selection is feasible but typically yields degraded performance.} utilizes the predictive distribution to select the top-$b_{t}$ most certain tokens for decoding. We employ the negative entropy of the predictive distribution as a generic \textit{surrogate} for confidence\footnote{Negative entropy serves as a smooth approximation to confidence-based metrics (e.g., margin confidence), particularly given that LLM distributions tend to be highly overconfident~\citep{sun2025largelanguagemodelsoverconfident}.}. The sampling procedure\footnote{We formulate this using deterministically greedy sampling ($\arg \max$), but it can be extended to stochastic categorical sampling.} from $\mathbf{x}_{t+1}$ to $\mathbf{x}_{t}$ is formulated as:

\begin{equation}
    x_{t,i} = \begin{cases}
        \arg\max\limits_{k \in \mathbb{X}} \ p_{t,i}^{k} & \text{if } i \in \mathcal{J}_{t} \\
        x_{t+1,i} & \text{otherwise}
    \end{cases},
    \label{eq:sampling procedure equation 1}
\end{equation}

where the components are defined as:
\begin{itemize}
    \item $p_{t,i} = \hat{p}_{\theta}(\cdot \mid x_{t+1,i} = \operatorname{M}, \mathbf{c}_{t,i}, \mathbf{s})$ is the predictive distribution at position $i$, given the prompt $\mathbf{s}$ and context $\mathbf{c}_{t,i}$. Here, $\mathbf{c}_{t,i}$ denotes the set of unmasked tokens that serve as the context for step $t$ for position $i$.
    \item $\mathcal{K}_{t} = \{i \mid x_{t+1,i} = \operatorname{M}\}$ is the set of indices currently masked.
    \item $\mathcal{J}_{t} \subset \mathcal{K}_{t}$ is the set of indices selected for decoding, determined by maximizing certainty:
    \begin{equation*}
        \mathcal{J}_{t} = \underset{\mathcal{S} \subset \mathcal{K}_{t}, |\mathcal{S}|=b_t}{\arg\max} \sum_{i \in \mathcal{S}} -H(p_{t,i}).
    \end{equation*}
\end{itemize}
Here, $H(p_{t,i}) = -\sum_{k \in \mathbb{X}} (p_{t,i}^{k} \log p_{t,i}^{k})$ denotes the Shannon entropy.

\section{Methodology}
\label{sec:Methodology}

\subsection{Motivation}
As illustrated in Eq. (\ref{eq:sampling procedure equation 1}), existing sampling strategies predominantly rely on the predictive distribution of the DLM at a specific position $i$ given the current state, denoted as $\hat{p}_{\theta}(\cdot \mid x_{t+1,i}, \mathbf{c}_{t,i}, \mathbf{s})$. This distribution not only determines the entropy statistics for the sampling index set $\mathcal{J}_t$ but also dictates the specific token values via probability mass. Since the input token $x_{t+1,i}$ is a uniform mask token for all indices in $\mathcal{K}_t$ and the prompt $\mathbf{s}$ is fixed, the predictive distribution is primarily conditioned on the variable context $\mathbf{c}_{t,i}$. Consequently, we refer to $\hat{p}_{\theta}(\cdot \mid x_{t+1,i}, \mathbf{c}_{t,i}, \mathbf{s})$ as the \textit{single-step predictive distribution}.

To quantify the quality of the sampling procedure governed by this single-step distribution, we introduce the Token Error Rate (TER) conditioned on the prompt:

\begin{definition}[Token Error Rate for Conditional Generation]~\citep{feng2025theoreticalbenefitlimitationdiffusion}
Given a target data distribution $p(\cdot \mid \mathbf{s})$ and the model's predictive distribution $\hat{p}(\cdot \mid \mathbf{s})$, the TER is defined as the expected negative log-likelihood per token:
\begin{align}
\log_2 \mathrm{TER}(\hat{p}(\cdot \mid \mathbf{s}); p(\cdot \mid \mathbf{s})) \defn -\frac{1}{N}\mathbb{E}_{\mathbf{x} \sim p(\cdot \mid \mathbf{s})}\big[\log \hat{p}(\mathbf{x} \mid \mathbf{s})\big],
\end{align}
where $N$ is the sequence length.
\end{definition}

This metric captures the alignment between the generated and target distributions. To analyze this at the token level, we decompose the TER using the decoding-time single-step predictive distribution $\hat{p}_{\theta}(x_{i} \mid \mathbf{c}_{\cdot,i}, \mathbf{s})$—where position $i$ is decoded into $x_{i}$ using a specific decoding-time context $\mathbf{c}_{\cdot,i}$. Assuming the target distribution factorizes independently across positions given the prompt for analytical tractability, we derive:

\begin{align}
\log_2 \mathrm{TER}(\hat{p}(\cdot|\mathbf{s}); p(\cdot|\mathbf{s})) 
&= -\frac{1}{N}\mathbb{E}_{\mathbf{x} \sim p(\cdot|\mathbf{s})}\left[\sum_{i=1}^{N} \log \hat{p}_{\theta}(x_{i}|\mathbf{c}_{\cdot,i},\mathbf{s})\right]\
&= -\frac{1}{N}\sum_{i=1}^{N}\mathbb{E}_{x_i \sim p(x_i|\mathbf{s})}\left[\log \hat{p}_{\theta}(x_{i}|\mathbf{c}_{\cdot,i},\mathbf{s})\right]
\end{align}

Minimizing this quantity is equivalent to minimizing the KL divergence between the target and predictive distributions, as the entropy of the true data distribution is constant. Thus, the optimization objective can be formulated as:
%\begin{align}
%    \min \log_{2} \text{TER}(\hat{p}(\cdot|\mathbf{s}); p(\cdot|\mathbf{s})) \equiv \min \frac{1}{N}\sum_{i=1}^{N} D_{\text{KL}}(p(x_i|\mathbf{s}) \| \hat{p}_{\theta}(x_{i}|\mathbf{c}_{\cdot,i},\mathbf{s})), \label{kl actual and estimate}
%\end{align}
\begin{align}
\log_2 \mathrm{TER}(\hat{p}(\cdot|\mathbf{s}); p(\cdot|\mathbf{s})) \propto \frac{1}{N}\sum_{i=1}^{N} D_{\text{KL}}(p(x_i|\mathbf{s}) \| \hat{p}_{\theta}(x_{i}|\mathbf{c}_{\cdot,i},\mathbf{s})), \label{kl actual and estimate}
\end{align}

Based on this analysis, we identify two critical limitations in existing sampling procedures:
\begin{itemize}[leftmargin=*]
    \item \textbf{Implicit approximation of target distribution using decoding-time single-step predictive distribution:} Existing sampling procedures implicitly assume that decoding-time single-step predictive distribution is a reliable surrogate for the target distribution, \textit{i.e.,} $p(x_i|\mathbf{s}) \approx \hat{p}_{\theta}(x_{i}|\mathbf{c}_{\cdot,i},\mathbf{s})$. However, such an approximation may be vulnerable when the decoding-time context $\mathbf{c}_{\cdot,i}$  in the current step is itself inaccurate, resulting in an unreliable and degraded predictive distribution. Correct sampling results cannot be guaranteed.
    \item \textbf{The TER for conditional generation cannot be governed appropriately.} Once an incorrect predictive distribution arising from an ineffective decoding-time context is taken into account, there is a large distribution discrepancy (\textit{i.e.,} $D_{\text{KL}}$ is large) between the single-step predictive distribution and the target distribution in Eq. (\ref{kl actual and estimate}). An appropriate level of TER thus cannot be guaranteed.
\end{itemize}

\subsection{Approximation of Target Distribution Using Marginalized Contexts}
To address these limitations, a straightforward approach would be to compute the target marginal distribution $p(x_i \mid \mathbf{s})$ directly, thereby bypassing the bias of any single context. While computing the true marginal is intractable, we propose approximating it by exploiting the intrinsic iterative nature of DLMs.

\textbf{Core insight.} Unlike autoregressive models, a DLM generates multiple predictive distributions for a masked position $i$ across the decoding trajectory (steps $T \to t$) before that position is finally unmasked. These intermediate predictions are conditioned on evolving contexts. We can approximate the true marginal distribution by integrating over these contexts. Mathematically, for any joint distribution, we can express the marginal distribution by introducing an auxiliary variable and integrating it out:
\begin{equation}
p(x_i \mid \mathbf{s}) = \int p(x_i, \mathbf{c} \mid \mathbf{s}) d\mathbf{c} = \int p(x_i \mid \mathbf{c}, \mathbf{s}) p(\mathbf{c} \mid \mathbf{s}) d\mathbf{c}.
\end{equation}

In the context of diffusion, the decoding trajectory provides a sequence of context states. We treat the sequence of contexts $\{\mathbf{c}_{T}, \dots, \mathbf{c}_{t}\}$ generated by the model as samples from the context distribution $p(\mathbf{c} \mid \mathbf{s})$. We approximate the integral via an ensemble average over the trajectory:

\begin{equation}
p(x_i \mid \mathbf{s}) \approx \overline{p}(x_i \mid \mathbf{s}) \triangleq \frac{1}{T-t+1}\sum_{k=0}^{T-t}\hat{p}_{\theta}(x_i \mid x_{T-k,i}=\operatorname{M}, \mathbf{c}_{T-k,i}, \mathbf{s}), \label{approximated computation}
\end{equation}
where $\overline{p}(x_i \mid \mathbf{s})$ is the approximated target distribution. This effectively performs a Monte Carlo estimation using $T-t$ correlated samples from the decoding process. Without the loss of generalization, one can integrate the approximated target distribution into existing sampling procedures in a plug-and-play manner, as follows:

\begin{equation}
    x_{t,i} = \begin{cases}
        \arg\max\limits_{k \in \mathbb{X}} \ p_{t,i}^{k} & \text{if } i \in \mathcal{J}_{t} \\
        x_{t+1,i} & \text{otherwise}
    \end{cases}, \quad \text{with} \quad
    \begin{aligned}
        &p_{t,i} = \overline{p}(x_i \mid \mathbf{s}), \\
        &\mathcal{J}_{t} = \underset{\mathcal{S} \subset \mathcal{K}_{t}, |\mathcal{S}|=b_t}{\arg\max} \sum_{i \in \mathcal{S}} -H(p_{t,i}).
    \end{aligned}
    \label{sampling procedure equation 2}
\end{equation}

Therefore, from the perspective of the overall decoding trajectory, Eq. (\ref{sampling procedure equation 2}) actually prioritizes those tokens that exhibit high \textit{confidence} and \textit{predictive consistency} across steps. In other words, if a token's distribution fluctuates dramatically due to changing contexts, the entropy of the averaged distribution $\overline{p}$ will be high, deprioritizing its selection. Conversely, tokens that are robust to context variations will be selected earlier. Instead, Eq. (\ref{eq:sampling procedure equation 1}) only considers tokens with high confidence in the current step, leading to an unreliable decoding process.

\textbf{Intuitive understandings.} First, by averaging over multiple contexts rather than relying on a single decoding-time context, our method becomes less sensitive to errors in any particular context. If some contexts $\mathbf{c}_{t,i}$ are noisy or incorrect, their negative impact is diluted by averaging with predictions from better contexts. Second, our method uses predictions that are already computed during the standard iterative decoding process, making it computationally efficient—essentially "free" in terms of additional forward passes through the model.

\textbf{Theoretical Analysis.}
We now demonstrate that our approximated target distribution connects theoretically to the single-step distribution and governs the sampling error bound.

\begin{proposition}[Connection via Mutual Information]
\label{prop:mutual_info_connection}
The approximated target distribution obtained by marginalizing over contexts is related to the decoding-time single-step predictive distribution through the conditional mutual information between contexts and token predictions:
\begin{equation}
H(x_i|\mathbf{s}) = H(x_i|\mathbf{c},\mathbf{s}) + I(x_i;\mathbf{c}|\mathbf{s}) \propto H(x_i|\mathbf{c}_{\cdot,i},\mathbf{s}) + I(x_i;\mathbf{c}|\mathbf{s}) \label{h plus i}
\end{equation}
where $I(x_i;\mathbf{c}|\mathbf{s})$ is the conditional mutual information between $x_i$ and $\mathbf{c}$ given $\mathbf{s}$, and $\mathbf{c}_{\cdot,i}$ is the specific decoding-time context.
\label{Proposition 1}
\end{proposition}

\begin{proof}
Since entropy directly measures the uncertainty of distributions and determines the selection of tokens to decode in DLMs (tokens with lower entropy are typically decoded first), we analyze the relationship through entropy decomposition. We start with the chain rule for conditional entropy applied in two different orders:
\begin{align}
H(x_i,\mathbf{c}|\mathbf{s}) = H(\mathbf{c}|\mathbf{s}) + H(x_i|\mathbf{c},\mathbf{s}) = H(x_i|\mathbf{s}) + H(\mathbf{c}|x_i,\mathbf{s})
\end{align}

Rearranging yields:
\begin{align}
H(x_i|\mathbf{s}) &= H(x_i|\mathbf{c},\mathbf{s}) + \underbrace{[H(\mathbf{c}|\mathbf{s}) - H(\mathbf{c}|x_i,\mathbf{s})]}_{I(x_i;\mathbf{c}|\mathbf{s})}= H(x_i|\mathbf{c},\mathbf{s}) + I(x_i;\mathbf{c}|\mathbf{s})
\end{align}

For the specific decoding-time context $\mathbf{c}_{\cdot,i}$, we have $H(x_i|\mathbf{c},\mathbf{s}) \approx H(x_i|\mathbf{c}_{\cdot,i},\mathbf{s})$ when $\mathbf{c}_{\cdot,i}$ is a sufficiently representative sample from the context distribution.
\end{proof}

Proposition \ref{prop:mutual_info_connection} uncovers two advantages of our proposed approximated target distribution over contexts for sampling procedures:  \begin{itemize}[leftmargin=*]
    \item\textbf{Additional consideration of statistical dependencies.} When contexts are informative about $x_i$ (high mutual information), our method captures additional statistical dependencies through the $I(x_i;\mathbf{c}|\mathbf{s})$ term, leading to better-calibrated predictions.
    \item \textbf{Flexible connection with the existing sampling procedure.} Our approach exhibits flexibility—when $x_i$ and $\mathbf{c}$ are conditionally independent given $\mathbf{s}$, we have $I(x_i;\mathbf{c}|\mathbf{s}) = 0$, and our method naturally degrades to the existing single-context method with $H(x_i|\mathbf{s}) = H(x_i|\mathbf{c}_{\cdot,i},\mathbf{s})$. This shows that our marginalization approach is a strict generalization that adapts to the underlying statistical structure of the problem.
\end{itemize}

Moreover, we uncover that our proposed approximated target distribution over contexts can directly govern sampling error upper bound based on Proposition \ref{Proposition 1}. We first recall the sampling error upper bound of DLM for a general non-conditional sampling (generation) process:
\begin{lemma}
\label{thm:main-result}\citep{li2025convergencetheorydiffusionlanguage}
For a well-trained DLM with any uniform mask size schedule, let $\overline{\mathbf{x}}\in \mathbb{X}^{N}$ be the ground-truth token sequence, and let $x_{i}$ and $x_{-i}$ be each token at position $i$ and the rest of the sequence,  the sampling error of the final output sequence $\hat{\mathbf{x}}\in \mathbb{X}^{N}$ of the sampling procedure satisfies 
\begin{align}\label{eq:main-result}
% \frac{1-T/L}{4T} \sum_{i=1}^L I(X_0^{(i)}; X_0^{(-i)}) + \varepsilon_\train \le 
% \mathbb{E}_{M_1,\dots,M_T}\big[\mathsf{KL}(p_{X_0\mid M_1,\dots,M_T}\parallel p_{Y_0\mid M_1,\dots,M_T})\big] 
% \mathsf{KL}(p_{X_0}\parallel p_{Y_0}) \leq
\mathbb{E}\big[\mathsf{KL}(p(\overline{\mathbf{x}})\parallel p(\hat{\mathbf{x}})\big] 
\le \frac{G}{T} \sum_{i=1}^N I(\overline{x}_{i}; \overline{x}_{-i})+ \varepsilon_\train.
\end{align}
Here, the expectation is taken over all possible mask sets throughout sampling processes.    $G$ is a constant depending on the mask schedule, and $T$ is the total number of decoding iterations. The term $\varepsilon_\train$ represents the inherent training error of the mask predictor. $I(\overline{x}_{i}; \overline{x}_{-i})$ represents the mutual information between each token and the rest of the sequence.
\end{lemma}

Note that the bound reveals that the potential sampling error depends on the intrinsic statistical coupling between tokens in the sequence. Specifically, sequences with higher mutual information between tokens exhibit larger potential sampling errors. Based on this well-constructed sampling error upper bound, we can make the following proposition for conditional generation:

\begin{proposition}[Sampling Error Governance via Marginalized Contexts]
\label{cor:marginalized-governance}
Under conditional generation with our approximated target distribution $\overline{p}(x_i|\mathbf{s})$ obtained by marginalizing over contexts, as $t \to 0$, the sampling error bound is governed by:
\begin{align}
\mathbb{E}\big[\mathsf{KL}(p(\overline{\mathbf{x}}|\mathbf{s})\parallel p(\hat{\mathbf{x}}|\mathbf{s}))\big]
&\le \frac{G}{T} \sum_{i=1}^N I(\overline{x}_{i}; \overline{x}_{-i}|\mathbf{s}) + \varepsilon_\train\
&\approx \frac{G}{T} \sum_{i=1}^N \left[\frac{1}{T-t+1}\sum_{k=0}^{T-t} I(\overline{x}_{i}; \mathbf{c}_{T-k,i}|\mathbf{s})\right] + \varepsilon_\train
\end{align}
This shows that our marginalization-based sampling procedure effectively governs the same sampling error bound at decoding time by using the averaged mutual information over the decoding trajectory. \label{propostion 2}
\end{proposition}

\begin{proof}
From Proposition \ref{Proposition 1}, we established that marginalizing over contexts yields:
\begin{equation}
H(\overline{x}_i|\mathbf{s}) = H(\overline{x}_i|\mathbf{c},\mathbf{s}) + I(\overline{x}_i;\mathbf{c}|\mathbf{s})
\end{equation}
During the DLM decoding process, at each iteration $T-k$ before position $i$ is decoded, we have a context $\mathbf{c}_{T-k,i}$ that progressively refines toward the ground-truth context $\overline{x}{-i}$, which can be viewed as the limiting case of this refinement process.
By marginalizing over the trajectory of contexts:
\begin{align}
\overline{p}(x_i|\mathbf{s}) &= \frac{1}{T-t+1}\sum_{k=0}^{T-t}\hat{p}_{\theta}(x_i|x_{T-k,i}=\text{M},\mathbf{c}_{T-k,i},\mathbf{s})
\end{align}

This marginalization effectively computes an averaged mutual information $\frac{1}{T-t+1}\sum_{k=0}^{T-t} I(\overline{x}_{i}; \mathbf{c}_{T-k,i}|\mathbf{s})$.

The mutual information $I(\overline{x}_{i}; \overline{x}_{-i} \mid \mathbf{s})$ represents the dependency of token $i$ on the full context. In our approximation, we substitute the full context $\overline{x}_{-i}$ with the evolving contexts $\{\mathbf{c}_{T-k,i}\}$ from the decoding trajectory.
Since $\overline{p}(x_i \mid \mathbf{s})$ averages the predictive distributions, the effective entropy used for selection is derived from this ensemble. As $t \to 0$, the context $\mathbf{c}_{t,i}$ converges toward the full unmasked sequence $\overline{x}_{-i}$. Therefore, the term $\frac{1}{T-t+1}\sum_{k=0}^{T-t} I(\overline{x}_{i}; \mathbf{c}_{T-k,i} \mid \mathbf{s})$ serves as a tractable lower-bound estimate of the true mutual information complexity, governing the sampling error dynamically during decoding as $t \to 0$:
\begin{equation}
\mathbb{E}\big[\mathsf{KL}(p(\overline{\mathbf{x}}|\mathbf{s})\parallel p(\hat{\mathbf{x}}|\mathbf{s}))\big] \le \frac{G}{T} \sum_{i=1}^N \left[\frac{1}{T-t+1}\sum_{k=0}^{T-t} I(\overline{x}_{i}; \mathbf{c}_{T-k,i}|\mathbf{s})\right] + \varepsilon_\train
\end{equation}
\end{proof}

Proposition \ref{propostion 2} uncovers two advantages of our proposed approximated target distribution over contexts for sampling procedures: \begin{itemize}[leftmargin=*]
    \item\textbf{Our method provides theoretical optimality guarantees.} Benefiting from the implicit introduction of mutual information as in Eq. (\ref{h plus i}), we can approximately maintain an equivalent sampling error bound structure with the ideal one. This transforms an inaccessible theoretical optimum into a practically computable objective, ensuring provably controlled sampling error at every decoding step.
    \item \textbf{An implicit difficulty-aware scheduling mechanism.} As our marginalization implicitly encodes the mutual information between each token and its evolving contexts, this may create a property: tokens with low mutual information (low sampling difficulty) are confidently decoded early, while tokens with high mutual information (high sampling difficulty) are strategically deferred. 
    From a perspective of the overall trajectory, the more contexts (more unmasked tokens) allow these difficult tokens to be decoded more easily (at a later stage).
\end{itemize}

\subsection{Coherent Contextual Decoding with Adaptive Sampling}
%For Eq. (\ref{sampling procedure equation 2}), if we straightforwardly implement the proposed approximated target distribution for a practical sampling procedure, there will be two limitations. (i) \textbf{Storing predictive distributions of all mask tokens is expensive.} For Eq. (\ref{approximated computation}), one would require to store a maximum of $(T-t)\times N \times |\overline{\mathbb{X}}|$ probability masses. In the scenario of long-sentence generation (\textit{i.e.,} $N$ is large), such an additional cost may not be acceptable for the desire for efficient decoding, especially in the limited memory of edge devices. (ii) \textbf{ The information usability of predictive distributions across the overall $T-t$ diffusion steps is not uniform.} It is obvious that the predictive distribution may not be very informative in the early diffusion steps with limited contexts, especially for those semantically ambiguous tokens. These early-stage predictions may be noisy results in Monte Carlo sampling, which in turn degrades the effectiveness of approximated target distributions. 
While Eq. (\ref{sampling procedure equation 2}) provides a theoretically principled approach, directly implementing the proposed approximated target distribution presents two practical challenges:
(i) \textbf{Non-uniform information quality across diffusion steps.} Commonly, the predictive distribution may not be very informative in the early diffusion steps due to the limited number of contexts, especially for those semantically ambiguous tokens. These early-stage predictions may be regarded as noisy results in Monte Carlo sampling, which degrades the effectiveness of approximated target distributions. 
(ii) \textbf{Prohibitive memory requirements for storing predictive distributions.} Implementing Eq. (\ref{approximated computation}) naively requires storing up to $(T-t)\times N \times |\overline{\mathbb{X}}|$ probability values, \textit{i.e.,} the full distribution for every masked token at every iteration. For long-sequence generation where $N$ is large, this memory overhead becomes prohibitive, particularly for deployment on memory-constrained edge devices where efficient decoding is paramount.

To address these practical challenges, we introduce a sliding-window historical buffer that maintains only the \textit{most recent and informative predictive distributions}. Specifically, we define a historical buffer $\mathcal{H}_t$ at iteration $t$ that stores the predictive distributions from the most recent $d$ iterations (except for the current iteration $t$) with only the top-$V$ most confident tokens at each iteration. Formally, $\mathcal{H}_t$ can be represented as follows
%\begin{equation}
%\mathcal{B}_t = \bigcap_{j=1}^{\min(d, T-t)} \mathcal{D}_{t+j}, \operatorname{where}\quad\mathcal{D}_{t+j} = {(i, \hat{p}_{\theta}(x_i|\mathbf{c}_{t+j,i}, \mathbf{s})) : i \in \arg \operatorname{top}_V \{-H(\hat{p}_{\theta}(x_i|\mathbf{c}_{t+j,i}, \mathbf{s}))\}_{i \in \mathcal{K}_{t+j}}}. \label{history}
%\end{equation}
%$\mathcal{D}_{t+j}$ is a set of position-distribution pairs of top-$V$ confident masked tokens. Note that the intersection operation collects all pairs across iterations, allowing the same position $i$ to appear multiple times in $\mathcal{B}_t$ if it was among the top-$V$ confident positions in multiple iterations.
$$\mathcal{H}_t = \{(i, j, \hat{p}_{\theta}(x_i|\mathbf{c}_{t+j,i}, \mathbf{s})) : i \in \mathcal{I}_t, j \in \{1, ..., \min(d, T-t)\}\}$$
%\begin{equation}
%\mathcal{H}_t = {(i, j, \hat{p}_{\theta}(x_i|\mathbf{c}_{t+j,i}, \mathbf{s})) : i \in \mathcal{I}_t, j \in {1, ..., \min(d, T-t)}},
%\label{history}\nonumber
%\end{equation}
\begin{equation}
    \mathcal{I}_t = \bigcap_{j=1}^{\min(d, T-t)} \{ i : i \in \underset{\mathcal{S} \subset \mathcal{K}_{t+j}, |\mathcal{S}|=V}{\arg\max} \sum_{i \in \mathcal{S}} -H(\hat{p}_{\theta}(x_i|\mathbf{c}_{t+j,i}, \mathbf{s})) \}
\label{consistent_positions}
\end{equation}
where $\mathcal{I}_t$ identifies mask token positions that appear in the top-$V$ confident sets across all recent $d$ iterations (intersection of position indices), which ensures  maximized effective context observations in the later computation of approximated target distribution. 
Based on $\mathcal{I}_t$, $\mathcal{H}_t$ maintains these selected token position indices, iteration indices, and corresponding predictive distribution in specific iteration index. $\mathcal{H}_t$ can be continually updated with iterations, and it always track consistently confident tokens in the most recent $d$ iterations in a sliding-window manner.

% This operation also  marginalization is applied selectively to consistently confident tokens with maximized effective context observations. 
At the current iteration $t$, we also identifies mask token positions that appear both in the current top-$V$ set and in the historical buffer. Finally, we can obtain consistently confident tokens with maximized effective context observations.
Formally, the position set $\mathcal{I}_t^c$ of these tokens can be formulated as follows
%those that demonstrate prediction stability across multiple iterations. This dual filtering prevents degradation from uninformative predictions while maximizing the benefits of marginalization for tokens that truly benefit from multiple 
\begin{equation}
\mathcal{I}_t^{c} = \{ i : i \in \underset{\mathcal{S} \subset \mathcal{K}_{t}, |\mathcal{S}|=V}{\arg\max} \sum_{i \in \mathcal{S}} -H(\hat{p}_{\theta}(x_i|\mathbf{c}_{t+j,i}, \mathbf{s})) \} \cap \{i : (i, \cdot,\cdot) \in \mathcal{H}_t \}.
\end{equation}

According to $\mathcal{I}_t^{c}$, we can obtain a current buffer $\mathcal{H}_t^{c}= \{(i, j, \hat{p}_{\theta}(x_i|\mathbf{c}_{t+j,i}, \mathbf{s})) : i \in \mathcal{I}_t^c, j \in \{0, ..., \min(d, T-t)\}\}$, which incorporates the history buffer and current predictions, to compute the approximated target distribution $\overline{p}(x_i|\mathbf{s})$ as in Eq. (\ref{approximated computation}) (where the number of Monte Carlo sampling is $d+1$ for each selected token).
The sampling procedure then becomes:
\begin{equation}
    x_{t,i} = \begin{cases}
        \arg\max\limits_{k \in \mathbb{X}} \ p_{t,i}^{k} & \text{if } i \in \mathcal{J}_{t} \\
        x_{t+1,i} & \text{otherwise}
    \end{cases}, \quad \text{with} \quad
    \begin{aligned}
        &p_{t,i} = \overline{p}(x_i \mid \mathbf{s}), i \in \mathcal{I}_t^{c} \\
        &\mathcal{J}_{t} = \underset{\mathcal{S} \subset \mathcal{I}_t^{c}, |\mathcal{S}|=b_t}{\arg\max} \sum_{i \in \mathcal{S}} -H(p_{t,i})
    \end{aligned}
    \label{pratical decoding}
\end{equation}

%\begin{equation}
%x_{t,i} = \begin{cases}
%x_{t+1,i} & i \notin \mathcal{J}_{t}, \operatorname{whre}  \mathcal{J}_{t} = \arg \operatorname{top}_{b_t}\{-H(p_{t,i})\}_{i \in \mathcal{I}_{t}^{c}}\\
%\arg\max \{p_{t,i}^{k}\}_{k=1}^{|\mathbb{X}|} & i \in \mathcal{J}_{t}, \quad p_{t,i}=\overline{p}(x_i|\mathbf{s}), i \in \mathcal{I}_t^{c}\\
%\end{cases}
%\label{pratical decoding}
%\end{equation}

\textbf{Discussions.} First, the sliding-window buffer reduces memory complexity from $O((T-t) \times N \times |\mathbb{X}|)$ to $O(d\times V \times |\mathbb{X}|)$, where $d \ll T-t$ and $V \ll N$. By storing only the most recent $d$ iterations and top-$V$ confident tokens per iteration, the approach becomes feasible for long-sequence generation and deployment on memory-constrained devices. Second, by selecting only the top-$V$ most confident predictions (lowest entropy) at each iteration, the buffer automatically filters out noisy, uninformative predictions from early diffusion steps. This ensures that the marginalization only incorporates high-quality, reliable distributions, preventing degradation from poorly-informed early-stage predictions while retaining the benefits of ensemble averaging over multiple context realizations.

\textbf{Adaptive Sampling Budget.} From Eq. (\ref{pratical decoding}), 
we can find two interesting properties for the historical buffer-driven implementation of our proposed approximated target distribution: (i) The selected tokens in $\mathcal{I}_t^{c}$ is consistently confident tokens, which may be more reliable to be candidates for decoding compared with those tokens selected from single-step confidence in Eq. (\ref{eq:sampling procedure equation 1}). (ii) The number of these selected tokens, \textit{i.e.,} $|\mathcal{I}_{t}^{c}|$ will dynamically vary from $0$ to $V$, according to the context sensitivity of generative contents in a diffusion period. For example, $|\mathcal{I}_{t}^{c}|$ tends to be large during context-insensitive generation (\textit{e.g.,} inherent templates or formulaic phrases), while it becomes smaller for context-sensitive content requiring careful semantic disambiguation.

These observations, particularly the second property, stand in stark contrast to existing sampling procedures that employ a uniform sampling budget $b_t$ at each diffusion step. This raises an intriguing possibility: \textit{can we leverage this natural variation to extend the existing uniform sampling budget scheme to a context-dynamically varied one?}
If we do so, there may be two merits. First, semantically ambiguous tokens (high context sensitivity) automatically receive smaller sampling budgets, allowing more iterations to establish proper contexts before decoding. Second, context-insensitive regions can be decoded more aggressively with larger budgets, reducing the overall number of diffusion steps required, which potentially accelerates the diffusion process. The second point is highly related to the generation of the end-of-sentence (EOS) that degrades the efficiency of the decoding process. Please refer to the emprical example in Figure \ref{fig:adaptive sampling} and analyses in sec. \ref{Analyses of Adaptive Sampling Budget}.

Formally, a context-dynamically varied sampling budget scheme can be represented based on Eq. (\ref{pratical decoding}) as follows:
\begin{equation}
    x_{t,i} = \begin{cases}
        \arg\max\limits_{k \in \mathbb{X}} \ p_{t,i}^{k} & \text{if } i \in \mathcal{J}_{t} \\
        x_{t+1,i} & \text{otherwise}
\end{cases}
\end{equation}
where $p_{t,i}=\overline{p}(x_i|\mathbf{s})$ for $i \in \mathcal{I}_t^{c}$, and the decoding set $\mathcal{J}_t$ is determined adaptively:

\begin{equation}
\mathcal{J}_{t} = \begin{cases}
\mathcal{I}_{t}^{c} & \text{if } |\mathcal{I}_{t}^{c}| \leq b_{t} \\
\underset{\mathcal{S} \subset \mathcal{I}_t^{c}, |\mathcal{S}|=b_t}{\arg\max} \sum_{i \in \mathcal{S}} -H(p_{t,i}) \cup \{i \in \mathcal{I}_{t}^{c} : H(p_{t,i}) < \epsilon\} & \text{if } |\mathcal{I}_{t}^{c}| > b_{t}%\operatorname{rank}(-H(p_{t,i})) > b_t \text{ and } H(p_{t,i}) < \epsilon} & \text{if } |\mathcal{I}_{t}^{c}| > b_{t}
\end{cases}
\end{equation}

As we can see,  when $|\mathcal{I}_{t}^{c}| > b_{t}$, we decode the top-$b_t$ tokens plus any additional tokens (ranked from $b_{t}+1$ to $|\mathcal{I}_{t}^{c}|$) whose entropy over marginalized contexts (\textit{i.e.,} marginal entropy) falls below a small threshold $\epsilon$. This ensures that all selected tokens exhibit both high confidence and predictive consistency across diffusion steps. When $|\mathcal{I}_{t}^{c}| \leq b_{t}$, we decode all $|\mathcal{I}_{t}^{c}|$ tokens directly.

\textbf{Discussion.} In this section, we present a comprehensive implementation of the Coherent Contextual Decoding method, along with an adaptive sampling budget extension designed to enhance both inference speed and generation quality. Furthermore, we need to mention that our decoding algorithm is compatible with existing inference optimization techniques, including remasking \cite{wang2025remaskingdiscretediffusionmodels} and block-wise decoding (semi-autoregressive) with KV-cache acceleration \cite{wu2025fastdllmtrainingfreeaccelerationdiffusion}.

\section{Experiments}
\label{sec:Experiment}

\subsection{Experimental Setup}
\begin{itemize} [leftmargin=*]
\item \textbf{Base Models.} We evaluate our approach using two prominent families of Diffusion Language Models (DLMs): LLaDA~\citep{nie2025llada} and Dream~\citep{dream2025}. Specifically, we utilize the \texttt{LLaDA-8B-Instruct} and \texttt{Dream-7B-Instruct} checkpoints. To ensure reproducibility, we employ their official open-source inference codebases and pretrained weights. We adhere strictly to the hyperparameter settings (e.g., diffusion step, temperature, unmasking algorithm) reported in their respective original papers for each benchmark. Both LLaDA and Dream employ a uniform sampling budget where the number of diffusion steps $T$ equals the maximum number of generated tokens $N$, implying a generation budget of $b_t = 1$ per step. We adopt this setting for all base comparisons.

\item \textbf{Baselines.} Our primary baseline is the standard sampling procedure with a uniform budget, as adopted by the original authors. Specifically, the Dream series utilizes negative entropy to select tokens, while the LLaDA series utilizes the maximum probability of single-step predictive distributions. We reproduce these baselines using their default configurations without modification. Note that the LLaDA series enable a block-wise (\textit{a.k.a.} semi-autoregressive) decoding manner that is preferably adopted by mathematical benchmarks with enhanced performance, \textit{e.g.}, GSM8K \citep{cobbe2021gsm8k} and MATH \citep{hendrycksmath2021}. To demonstrate the potential of our proposed method on such a semi-autoregressive scenario, we also utilize the block-wise decoding scheme for the LLaDA series on  GSM8K and MATH benchmarks.

\item \textbf{Datasets and Metrics.} We conduct evaluations across five standard benchmarks covering three distinct capabilities:
\begin{itemize} [leftmargin=*]
    \item \textbf{Mathematical Reasoning:} GSM8K \citep{cobbe2021gsm8k} and MATH \citep{hendrycksmath2021}.
    \item \textbf{Code Generation:} HumanEval \citep{chen2021codex} and MBPP \citep{austin2021structured}.
    \item \textbf{Planning:} The trip planning benchmark \cite{zheng2024naturalplanbenchmarkingllms} contains 1600 examples with different difficulties.
\end{itemize}

All tasks are evaluated in a zero-shot setting, with the exception of the Trip benchmark. We utilize the standardized \texttt{lm-eval} harness \citep{eval-harness} to compute accuracy. Beyond accuracy, we report the average number of decoding steps required. As DLM inference is typically memory-bound, the reduction in decoding steps serves as a direct proxy for inference latency speedup. All experiments were conducted on a single powerful GPU equipped with 80GB of memory and high memory bandwidth.
\end{itemize}

\subsection{Implementation Details}

\begin{itemize} [leftmargin=*]

\item \textbf{Proposed Method (CCD).} Our proposed approach utilizes a historical buffer to approximate the target distribution. This module is designed to be plug-and-play, compatible with existing sampling procedures. We leverage the base model's native uncertainty metric (negative entropy for Dream; max probability for LLaDA) to identify the top-$V$ most confident tokens at each iteration. These tokens construct both the current buffer and the sliding-window historical buffer. The CCD sampling procedure naturally generalizes the standard sampling methods; specifically, when the history length $d=1$ and the candidate set size $V=b_t$, our method reduces mathematically to the default sampling schemes of Dream and LLaDA.

\item \textbf{Hyperparameters.} The performance of CCD is governed by two key hyperparameters: the number of confident tokens retained per iteration ($V$) and the history length ($d$). Unless otherwise stated, we set $V=4$. The parameter $d$ determines the effective context window for the approximated target distribution (Eq. \ref{approximated computation}). Since the predictive stability varies across base models, we tune $d$ specifically for each architecture: we set $d=3$ for the Dream series and $d=2$ for the LLaDA series.

\item \textbf{Dynamic Scheduling (CCD-DS).} For our context-\underline{d}ynamically varied \underline{s}ampling variant (CCD-DS), the threshold $\epsilon$ controls the sensitivity to marginal entropy, determining when to accelerate generation. Tuning a fixed $\epsilon$ across diverse benchmarks is non-trivial. To address this, we introduce a heuristic based on prediction stability to avoid explicit tuning of $\epsilon$. Specifically, at position $i$, if the token index with the maximum probability remains consistent across the $d+1$ iterations in the buffer, we infer that the marginal entropy is sufficiently low (implicitly assuming $H(p_{t,i}) < \epsilon$). This stability criterion allows for adaptive acceleration without manual threshold specification.
\end{itemize}

\subsection{Main results}
\textbf{Performance Improvement with CCD.} We begin by evaluating the effectiveness of our proposed CCD method under a uniform sampling budget of $b_t=1$. As shown in Table \ref{tab:main_table_combined}, the deployment of CCD consistently enhances performance metrics across all five benchmarks for both the LLaDA-8B-Instruct and Dream-7B-Instruct models. These gains are particularly pronounced in complex reasoning tasks; for example, the Dream model achieves a score increase of $+4.65$ on HumanEval and $+1.83$ on the Trip Plan benchmark. This validates that our approximated target distribution effectively guides the model toward higher-quality outputs. Notably, CCD also demonstrates robust compatibility with the block-wise decoding scheme of the LLaDA series on the GSM8K and MATH benchmarks.

\textbf{Enhanced Efficiency and Performance with CCD-DS.} Upon applying the adaptive CCD-DS strategy, we observe a significant inference speedup while simultaneously maintaining or even improving performance.
\begin{itemize}[leftmargin=*]
    \item \textbf{Inference Speedup}: CCD-DS substantially reduces the required diffusion steps. For the Dream model, we observe efficiency gains as high as $3.78\times$ on MBPP and $3.48\times$ on Trip Plan. Similarly, LLaDA achieves a $2.27\times$ speedup on the Trip Plan task.
    \item \textbf{Extra Performance Improvement}: Crucially, this acceleration does not degrade generation quality. In fact, CCD-DS frequently outperforms the fixed-step baseline CCD. For instance, on the Trip Plan benchmark, the Dream model with CCD-DS achieves a score of $19.01$ (vs. $16.93$ with CCD), delivering superior performance alongside a nearly fourfold increase in inference speed.
\end{itemize}

\captionsetup{skip=5pt}
\begin{table*}[h]
  \caption{
    \textbf{Performance of LLaDA 8B and Dream 7B with our proposed method } on 5 benchmarks. Note that 
the diffusion steps are varied across benchmarks and base models, as these settings typically correspond to the best performance on a specific benchmark. Therefore, we follow the base models' default settings without tuning.
  }
  \label{tab:main_table_combined}
  \centering
  \renewcommand{\arraystretch}{1.1}
  \resizebox{1.0\textwidth}{!}{
  \begin{tabular}{c|c|l l|l|c|l l|l}
     \toprule
    \multirow{2}{*}{\bf Task} & \multirow{2}{*}{\bf Method} & \multicolumn{2}{c}{\bf Inference Efficiency} & \multicolumn{1}{|c}{\bf Performance} & \multicolumn{1}{|c|}{\multirow{2}{*}{\bf Method}} & \multicolumn{2}{c}{\bf Inference Efficiency} & \multicolumn{1}{|c}{\bf Performance} \\
    \cmidrule(lr){3-4} \cmidrule(lr){5-5}  \cmidrule(lr){7-8} \cmidrule(lr){9-9} 
    & & {\bf Diffusion steps$\downarrow$} & {\bf Gains$\uparrow$} & {\bf Score$\uparrow$} & & {\bf Diffusion steps$\downarrow$} & {\bf Gains$\uparrow$} & {\bf Score$\uparrow$}\\
  \midrule
   \rowcolor{gray!15} \multicolumn{9}{c}{\bf Mathematics Reasoning}\\
  \midrule
    \multirow{3}{*}{GSM8K}
                           & \textcolor{gray}{LLaDA Instruct} & \textcolor{gray}{512} & \textcolor{gray}{1.00$\times$} & \textcolor{gray}{74.30}
                           & \textcolor{gray}{Dream Instruct} & \textcolor{gray}{256} & \textcolor{gray}{1.00$\times$} &  \textcolor{gray}{81.01} \\
                          
                           & ~+~CCD  & 512 & 1.00$\times$  & \textbf{75.30}$_{\textcolor{LightGreen}{+1.00}}$
                           & ~+~CCD & 256 & 1.00$\times$ & 82.26$_{\textcolor{LightGreen}{+1.25}}$\\
                           &\cellcolor{blue!10} ~+~CCD-DS & \cellcolor{blue!10}393.0$_{\textcolor{LightGreen}{-119.0}}$ & \cellcolor{blue!10}1.31$\times_{\textcolor{LightGreen}{+0.31}}$& \cellcolor{blue!10}75.22$_{\textcolor{LightGreen}{+0.92}}$ 
                          & \cellcolor{blue!10}~+~CCD-DS & \cellcolor{blue!10}141.2$_{\textcolor{LightGreen}{-114.8}}$ & \cellcolor{blue!10}1.82$\times_{\textcolor{LightGreen}{+0.82}}$ & \cellcolor{blue!10}\textbf{82.51} $_{\textcolor{LightGreen}{+1.50}}$ \\

  \midrule
    \multirow{3}{*}{Math}
                          & \textcolor{gray}{LLaDA Instruct} & \textcolor{gray}{512} & \textcolor{gray}{1.00$\times$} & \textcolor{gray}{37.00}
                           & \textcolor{gray}{Dream Instruct} & \textcolor{gray}{512} & \textcolor{gray}{1.00$\times$} &  \textcolor{gray}{40.90} \\
                          
                               & ~+~CCD & 512  & 1.00$\times$& \textbf{37.20}$_{\textcolor{LightGreen}{+0.20}}$ 
                          & ~+~CCD & 512 &1.00$\times$ & \textbf{41.20} $_{\textcolor{LightGreen}{+0.30}}$ \\
                           & \cellcolor{blue!10}~+~CCD-DS & \cellcolor{blue!10}378.2 $_{\textcolor{LightGreen}{-133.8}}$ & \cellcolor{blue!10}1.35$\times_{\textcolor{LightGreen}{+0.35}}$& \cellcolor{blue!10}\textbf{37.20}$_{\textcolor{LightGreen}{+0.20}}$ 
                          & \cellcolor{blue!10}~+~CCD-DS & \cellcolor{blue!10}340.2$_{\textcolor{LightGreen}{-171.8}}$ & \cellcolor{blue!10}1.58$\times_{\textcolor{LightGreen}{+0.58}}$ & \cellcolor{blue!10}\textbf{41.20} $_{\textcolor{LightGreen}{+0.30}}$ \\
                          
  \midrule
  \rowcolor{gray!15} \multicolumn{9}{c}{\bf Code Generation}\\
    \midrule
   \multirow{3}{*}{HumanEval}
                                    & \textcolor{gray}{LLaDA Instruct} & \textcolor{gray}{512} & \textcolor{gray}{1.00$\times$} & \textcolor{gray}{36.50}
                           & \textcolor{gray}{Dream Instruct} & \textcolor{gray}{768} & \textcolor{gray}{1.00$\times$} &  \textcolor{gray}{52.66} \\
                          
                           & ~+~CCD  & 512 & 1.00$\times$  & \textbf{38.41}$_{\textcolor{LightGreen}{+1.91}}$
                           & ~+~CCD & 768 & 1.00$\times$ & \textbf{57.31}$_{\textcolor{LightGreen}{+4.65}}$\\
                           & \cellcolor{blue!10}+~CCD-DS & \cellcolor{blue!10}332.0 $_{\textcolor{LightGreen}{-180.0}}$ & \cellcolor{blue!10}1.54$\times_{\textcolor{LightGreen}{+0.54}}$& \cellcolor{blue!10}38.40$_{\textcolor{LightGreen}{+1.90}}$ 
                          & \cellcolor{blue!10}~+~CCD-DS & \cellcolor{blue!10}253.2$_{\textcolor{LightGreen}{-514.8}}$ & \cellcolor{blue!10}3.04$\times_{\textcolor{LightGreen}{+2.04}}$ & \cellcolor{blue!10}56.71$_{\textcolor{LightGreen}{+4.05}}$ \\
  \midrule
    \multirow{3}{*}{MBPP}
                                & \textcolor{gray}{LLaDA Instruct} & \textcolor{gray}{256} & \textcolor{gray}{1.00$\times$} & \textcolor{gray}{39.20}
                           & \textcolor{gray}{Dream Instruct} & \textcolor{gray}{1024} & \textcolor{gray}{1.00$\times$} &  \textcolor{gray}{58.00} \\
                          
                           & ~+~CCD  & 256 & 1.00$\times$  & 39.20
                           & ~+~CCD & 1024 & 1.00$\times$ & 58.00\\
                           & \cellcolor{blue!10}~+~CCD-DS & \cellcolor{blue!10}211.20 $_{\textcolor{LightGreen}{-44.8}}$ & \cellcolor{blue!10}1.24$\times_{\textcolor{LightGreen}{+0.24}}$& \cellcolor{blue!10}\textbf{39.20}$_{\textcolor{LightGreen}{+0.00}}$ 
                          & \cellcolor{blue!10}~+~CCD-DS & \cellcolor{blue!10}270.20$_{\textcolor{LightGreen}{-753.80}}$ & \cellcolor{blue!10}3.78$\times_{\textcolor{LightGreen}{+2.78}}$ & \cellcolor{blue!10}\textbf{58.00} $_{\textcolor{LightGreen}{+0.00}}$ \\

    \midrule                        
 \rowcolor{gray!15} \multicolumn{9}{c}{\bf Planing}\\
  \midrule
    \multirow{3}{*}{Trip Plan}
                        & \textcolor{gray}{LLaDA Instruct} & \textcolor{gray}{256} & \textcolor{gray}{1.00$\times$} & \textcolor{gray}{10.40}
                           & \textcolor{gray}{Dream Instruct} & \textcolor{gray}{256} & \textcolor{gray}{1.00$\times$} &  \textcolor{gray}{15.10} \\
                          
                           & ~+~CCD  & 256 & 1.00$\times$  & 10.80$_{\textcolor{LightGreen}{+0.40}}$
                           & ~+~CCD & 256 & 1.00$\times$ & 16.93$_{\textcolor{LightGreen}{+1.83}}$  \\
                           & \cellcolor{blue!10}~+~CCD-DS & \cellcolor{blue!10}112.5 $_{\textcolor{LightGreen}{-143.5}}$ & \cellcolor{blue!10}2.27$\times_{\textcolor{LightGreen}{+1.27}}$& \cellcolor{blue!10}\textbf{11.50}$_{\textcolor{LightGreen}{+1.10}}$ 
                          & \cellcolor{blue!10}~+~CCD-DS & \cellcolor{blue!10}75.20$_{\textcolor{LightGreen}{-180.20}}$ & \cellcolor{blue!10}3.48$\times_{\textcolor{LightGreen}{+2.48}}$ & \cellcolor{blue!10}\textbf{19.01} $_{\textcolor{LightGreen}{+3.91}}$ \\
  \bottomrule
  \end{tabular}
  }
\end{table*}

\subsection{Analysis of Adaptive Sampling Budget}
\label{Analyses of Adaptive Sampling Budget}
\textbf{Decoding Efficiency Bottleneck: Generation Plateau Period of DLMs.}
As illustrated in Figure \ref{fig:adaptive sampling}(a), instead of discretely generating EOS tokens across the decoding process, we observe that the decoding process of DLMs actually suffers from continual generation periods of EOS tokens, \textit{i.e.,} multiple plateau periods in the increment of effective tokens (excluding EOS tokens). Since we cannot know the ground-truth number of generative responses beforehand, a more redundant number of mask tokens are typically adopted, resulting in the inevitable generation of these EOS tokens. For the existing uniform sampling budget scheme, it is obvious that the decoding efficiency is very low during these plateau periods.

In contrast, as illustrated in Figure \ref{fig:adaptive sampling}(b), our proposed method can adaptively obtain an ad-hoc sampling budget across the decoding process with maximum budgets set to $4$. Therefore, our proposed method can pass through the plateau periods more efficiently by decoding more EOS tokens in each step. Meanwhile, we observe that our proposed method can automatically obtain smaller sampling budgets (1 or 2) at the late stage of the decoding process, which ensures the accuracy of generation as this stage usually includes semantically ambiguous tokens.

\subsection{Hyperparameter Analysis}
\textbf{Impact of Buffer Size.} Our proposed method relies on the size of the historical buffer for balancing efficiency and effectiveness gains. As shown in Figure~\ref{fig:sub1}, our method consistently outperforms the baseline across all buffer sizes in terms of both performance and inference speed. The accuracy peaks at buffer size 4 (70\%), representing a 20.7\% improvement over the baseline (58\%). Simultaneously, our approach requires substantially fewer steps across all configurations, with the reduction ranging from 29.5\% (180.37 vs. 256 steps at buffer size 1) to 71.4\% (73.21 vs. 256 steps at buffer size 6). Notably, buffer size 4 achieves an optimal balance, delivering the highest accuracy while requiring only 95.54 steps—a 62.7\% reduction in diffusion step. This demonstrates that our method effectively leverages the historical buffer to accelerate decoding process while significantly enhance performance.

\textbf{Robustness to Temperature Coefficients.} As illustrated in Figure~\ref{fig:sub2}, our method demonstrates strong robustness across varying temperature coefficients. Both methods exhibit similar trends, with performance peaking at temperature 0.1 (52.66\% for baseline, 56.71\% for ours) and declining at extreme values. However, our method consistently outperforms the baseline across the entire temperature range, with improvements of 9.8\%, 7.7\%, 1.5\%, 9.0\%, and 2.0\% at temperatures 0, 0.1, 0.4, 0.7, and 1.0, respectively. The performance advantage is most pronounced at lower temperatures (0 and 0.1) and near the upper bound (0.7), suggesting our approach is particularly effective when exploration-exploitation balance is critical. Importantly, even at suboptimal temperatures (0.7 and 1.0), our method maintains meaningful gains, demonstrating robustness to hyperparameter choices.

\begin{figure}[h]
    \centering
    \begin{subfigure}{0.49\textwidth}
        \centering
        \includegraphics[width=\textwidth]{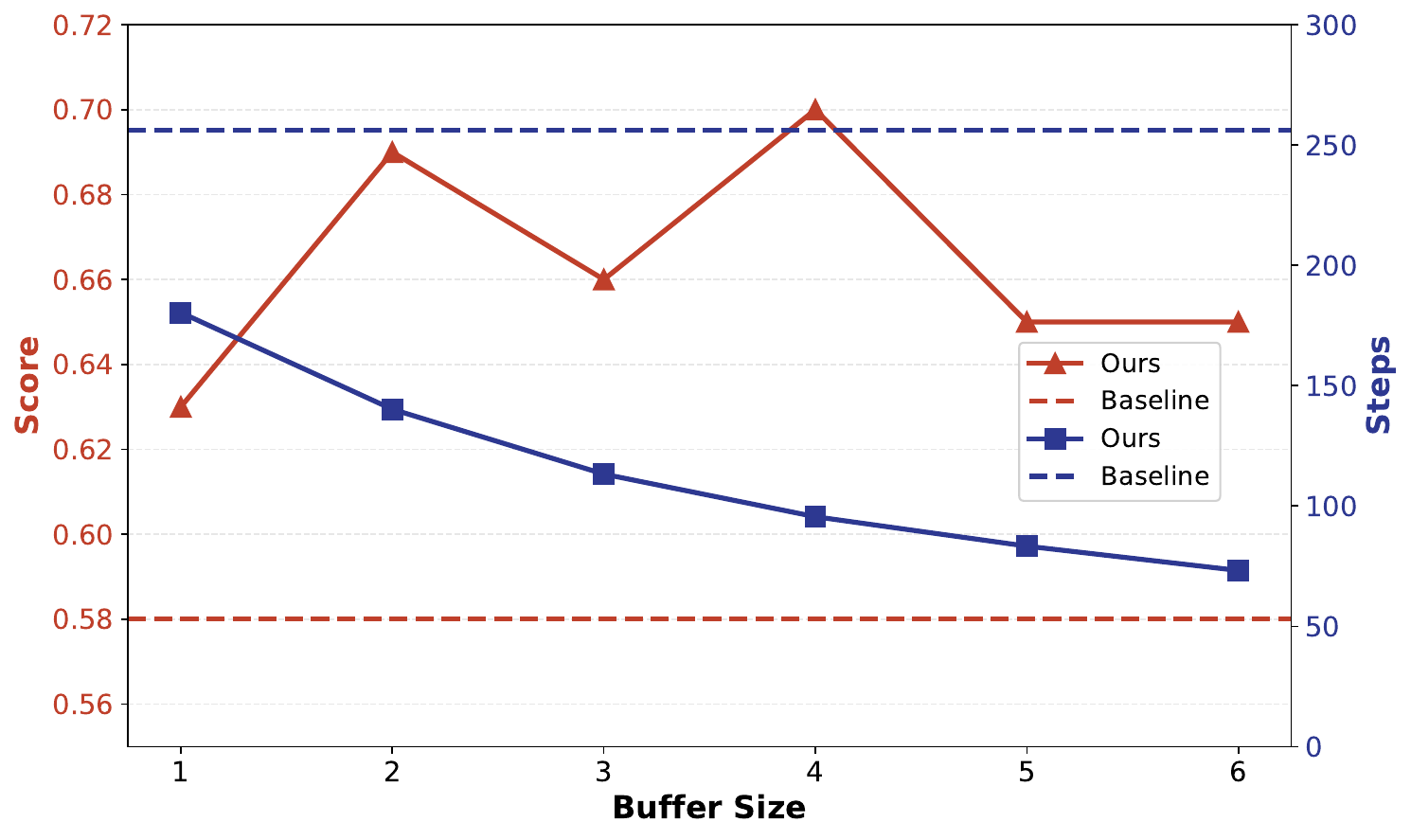}
        \caption{}
        \label{fig:sub1}
    \end{subfigure}
    \hfill
    \begin{subfigure}{0.49\textwidth}
        \centering
        \includegraphics[width=\textwidth]{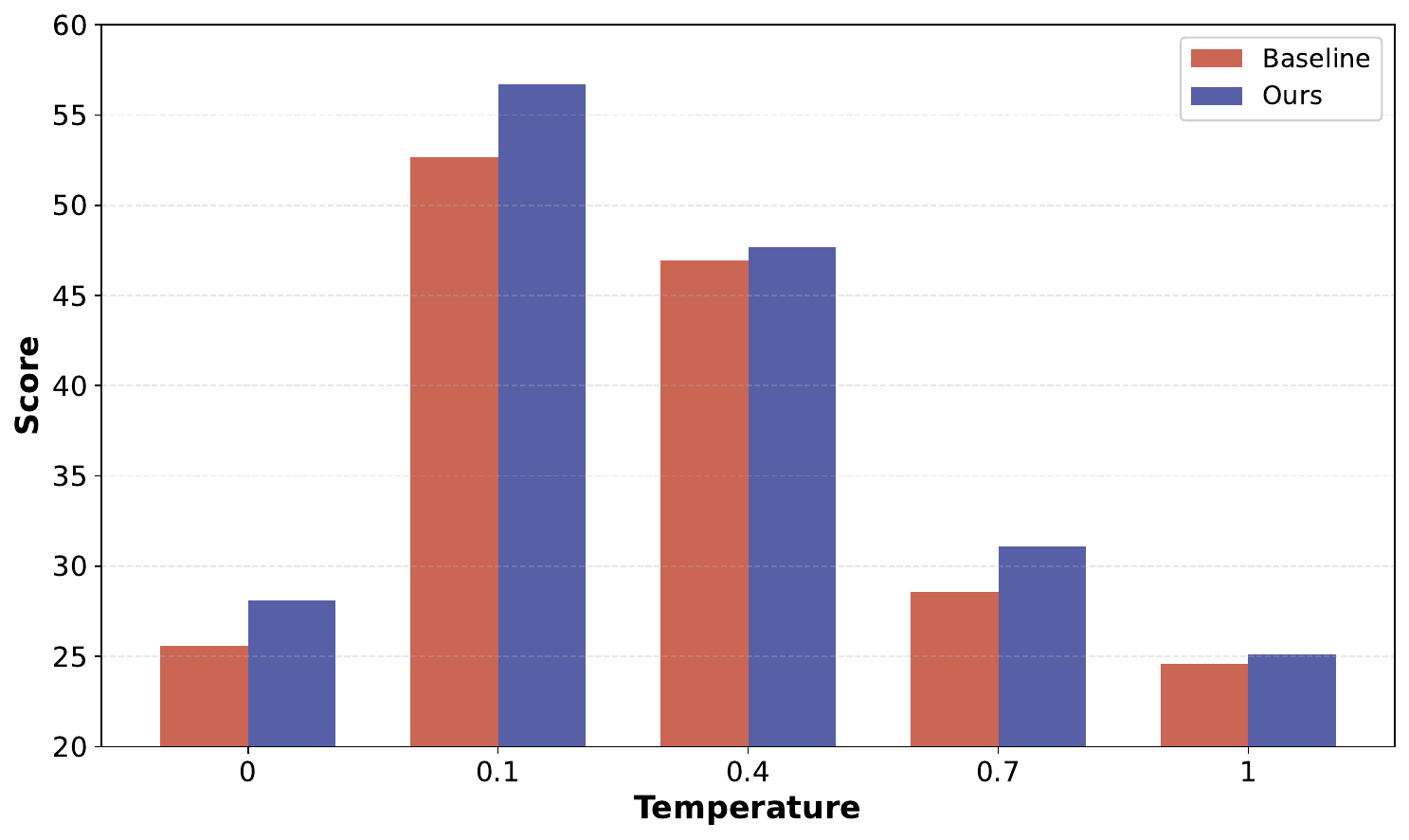}
        \caption{}
        \label{fig:sub2}
    \end{subfigure}
    \caption{Hyperparameter Analysis on buffer size and temperature coefficients using the Dream model. (a) The trade-off between score and computational steps as buffer size varies on the subset (City=3) of Trip benchmark. (b) Performance comparison across different temperature coefficients on the HumanEval benchmark.}
    \label{fig:main}
\end{figure}

\begin{table}[!t]
  \renewcommand{\arraystretch}{0.5}
\centering
\begin{tabular}{c}
\begin{tcolorbox}[title={\textbf{Task:} \emph{Lorie earns \$10 per hour. Karen earns twice what Lorie earns. How much does Karen earn in two days if she works 3 hours per day?
\quad \colorbox{red!30}{\texttt{Incorrect}}\quad \colorbox{green!30}{\texttt{Correct}}}\quad \colorbox{blue!10}{$\mathbf{\otimes}$} Mask token}, colback=white,colframe=black!10, coltitle=black, width=\textwidth]
\noindent \textbf{Baseline:}\; \texttt{Lorie earns \$10 per hour, so in 3 hours she earns \$10 x 3 = \$30.}\\ 
\texttt{Karen earns twice what Lorie earns, so in 3 hours she earns 2 x \$30 = \$60.} \\
\texttt{Karen works 3 hours per day, so in 2 days she works 3x 2 = 6 hours.} \\
\texttt{Therefore, in 2 days, Karen earns \$60 x 6 = \$360.}\tikz[baseline=-0.5ex]{\draw[red, dashed, ->, thick] (0,0) -- (0.6,0);} \textcolor{red}{\textit{It should be  \$60 x 2, rather than \$60 x 6}}\\
\texttt{The answer is:} \colorbox{red!30}{\texttt{360}}\\
\noindent\hdashrule{\textwidth}{0.5pt}{2pt}\\
\vspace{0.3em}
\noindent\colorbox{gray!20}{\parbox{\dimexpr\textwidth-2\fboxsep}{\centering\vspace{0.2em}\textbf{Illustration of diffusion intermediate process  when the generative trajectory starts to be separated}\vspace{0.2em}}}
\vspace{-0.8em}\\
\noindent\textcolor{gray}{\scriptsize token index \texttt{\hspace{2.3em}1\hspace{4.5em}2\hspace{3.3em}3\hspace{1.6em}4\hspace{3.em}5\hspace{1.1em}6\hspace{1.2em}7\hspace{1.4em}8\hspace{1.6em}9\hspace{1.4em}10\hspace{1.4em}11}}\\
\noindent \textbf{Baseline:}\; \texttt{Lorie earns \$10 per hour,} \colorbox{blue!10}{$\mathbf{\otimes}$} \colorbox{blue!10}{$\mathbf{\otimes}$} \colorbox{blue!10}{$\mathbf{\otimes}$} \colorbox{blue!10}{$\mathbf{\otimes}$} \colorbox{blue!10}{$\mathbf{\otimes}$} \textbf{...} \tikz[baseline=-0.5ex]{\draw[red, dashed, ->, thick] (0,0) -- (0.7,0);} \textcolor{red}{\textit{Step \#7}}
\begin{center}
\vspace{-0.1cm}
\quad\quad\quad\quad\quad\quad\quad\quad\quad\quad\quad\quad\quad
\tikz{\draw[->, thick, line width=1.pt] (0,0) -- (0,-0.5) node[right, pos=0.5] {\textit{diffusion, do sampling procedure in Eq.(\ref{eq:sampling procedure equation 1})}};}
\end{center}
\colorbox{gray!20}{\parbox{\dimexpr\textwidth-2\fboxsep}{\texttt{Top-1 confident index of single-step predictive distributions: \textcolor{red}{7th mask token}}}}
\begin{center}
\vspace{-0.1cm}
\tikz{\draw[->, thick, line width=1.pt] (0,0) -- (0,-0.5) node[right, pos=0.5] {\textit{decoding}};}
\end{center}
\vspace{-0.1cm}
\texttt{Lorie earns \$10 per hour, so} \colorbox{blue!10}{$\mathbf{\otimes}$} \colorbox{blue!10}{$\mathbf{\otimes}$} \colorbox{blue!10}{$\mathbf{\otimes}$} \colorbox{blue!10}{$\mathbf{\otimes}$} \textbf{...} \tikz[baseline=-0.5ex]{\draw[red, dashed, ->, thick] (0,0) -- (1,0);} \textcolor{red}{\textit{Step \#8}}
\end{tcolorbox}
\\[0.1em]
\begin{tcolorbox}[title={}, colback=white,colframe=blue!15, coltitle=black, width=\textwidth]
\noindent \textbf{Ours:}\; \texttt{Lorie earns \$10 per hour, so} \textcolor{red}{\underline{Karen}} \texttt{ earns twice that, which is \$10 x 2 = \$20 per hour.} \tikz[baseline=-0.5ex]{\draw[red, dashed, ->, thick] (0,0) -- (0.8,0);} \textcolor{red}{\textit{The generative trajectory starts to be different from the results of baseline at the "Karen" }}\\
\texttt{
If Karen works 3 hours per day, she earns \$20 x 3 = \$60 per day.} \\
\texttt{In two days, Karen earns \$60 x 2 = \$120.}\\
\texttt{The answer is:} \colorbox{green!30}{\texttt{120}}\\
\noindent\hdashrule{\textwidth}{0.5pt}{2pt}\\
\vspace{0.3em}
\noindent\colorbox{gray!20}{\parbox{\dimexpr\textwidth-2\fboxsep}{\centering\vspace{0.2em}\textbf{Illustration of diffusion intermediate process when the generative trajectory starts to be separated}\vspace{0.2em}}}
\vspace{-0.8em}\\
\noindent\textcolor{gray}{\scriptsize token index \texttt{\hspace{1.3em}1\hspace{3.5em}2\hspace{3em}3\hspace{2.1em}4\hspace{3.em}5\hspace{1.1em}6\hspace{1.2em}7\hspace{1.4em}8\hspace{1.6em}9\hspace{1.4em}10\hspace{1.4em}11}}\\
\noindent \textbf{Ours:}\; \texttt{Lorie earns \$10 per hour,} \colorbox{blue!10}{$\mathbf{\otimes}$} \colorbox{blue!10}{$\mathbf{\otimes}$} \colorbox{blue!10}{$\mathbf{\otimes}$} \colorbox{blue!10}{$\mathbf{\otimes}$} \colorbox{blue!10}{$\mathbf{\otimes}$} \textbf{...} \tikz[baseline=-0.5ex]{\draw[red, dashed, ->, thick] (0,0) -- (1,0);} \textcolor{red}{\textit{Step \#7}}
\begin{center}
\vspace{-0.1cm}
\quad\quad\quad\quad\quad\quad\quad\quad\quad\quad\quad\quad\quad\quad
\tikz{\draw[->, thick, line width=1.pt] (0,0) -- (0,-0.5) node[right, pos=0.5] {\textit{diffusion, do sampling procedure in Eq.(\ref{pratical decoding})}};}
\colorbox{gray!20}{\parbox{\dimexpr\textwidth-2\fboxsep}{\texttt{Top-1 confident index of approximated target distributions: \textcolor{red}{8th mask token}}}}

\end{center}
\begin{center}
\vspace{-0.2cm}
\tikz{\draw[->, thick, line width=1.pt] (0,0) -- (0,-0.5) node[right, pos=0.5] {\textit{decoding}};}
\end{center}
\vspace{-0.2cm}
\texttt{Lorie earns \$10 per hour,} 
\colorbox{blue!10}{$\mathbf{\otimes}$} 
\texttt{Karen}
\colorbox{blue!10}{$\mathbf{\otimes}$} \colorbox{blue!10}{$\mathbf{\otimes}$} \colorbox{blue!10}{$\mathbf{\otimes}$} \textbf{...} \tikz[baseline=-0.5ex]{\draw[red, dashed, ->, thick] (0,0) -- (0.5,0);} \textcolor{red}{\textit{Step \#8, reject single-step top-1 confident index}}
\end{tcolorbox}
\end{tabular}
\caption{Final outputs of different sampling procedures on an example of the GSK8K benchmark. We present the reason analysis of the different final outputs by tracking the trajectory separated points. }
\label{gsm8k}
\vspace{-0.5cm}
\end{table}

\subsection{Effectiveness Analysis By Generative Examples}

\textbf{Our proposed method can reject suboptimal single-step predictive distribution and determine a true trajectory.} As illustrated in Table \ref{gsm8k}, at the critical Step \#7, the baseline's single-step confidence metric selects position 7 for the token "so", a syntactic connector with high local confidence but low semantic value. This choice initiates a convoluted reasoning path where the model first calculates Lorie's 3-hour earnings (\$30), then misapplies this intermediate result to Karen's calculation, ultimately confusing daily earnings (\$60) with total hours worked (6) to produce the incorrect answer of \$360. In contrast, the proposed approximated target distribution identifies position 8 ("Karen") as optimal, rejecting the baseline's locally confident choice. By decoding "Karen" directly after "Lorie earns \$10 per hour,", our method establishes a cleaner logical flow that immediately transitions to Karen's hourly rate (\$20), leading to the correct calculation.

This example illustrates a fundamental advantage of our sampling strategy: the ability to distinguish between syntactic fluency and semantic importance. While single-step predictive distributions often favor high-confidence function words or connectors that ensure grammatical correctness, these choices may lead to suboptimal reasoning trajectories. Our method's use of the most recent iteration consistency as a filtering mechanism naturally prioritizes semantically pivotal tokens—those that determine the logical direction of problem-solving—over merely fluent transitions. This results in more direct reasoning paths that avoid unnecessary intermediate calculations and maintain clearer alignment between the problem statement and solution strategy, ultimately reducing cascading errors in multi-step reasoning tasks.

\textbf{More examples and advantages.} Additionally, we observe that our proposed method can not only automatically defer the decoding of semantically ambiguous tokens but also encourage more reasonable generative trajectories. Please refer to these examples and corresponding analyses in Tables  \ref{gsm8k2} and \ref{tab:trip example}.

%Key parameters for cache and self-speculative decoding include block 
%length, draft length, and cache refresh interval. Unless stated otherwise, we use a block length of 8, a cache refresh interval of 8 (meaning key-value states are updated every 8 decoded tokens), and a draft length $\in {3, 4, 5}$. The maximum generation length for all models is set to 256 tokens.

\section{Related Work}
\label{sec:Related Work}

\subsection{Diffusion Language Models}
The application of diffusion processes to discrete domains, specifically for text generation, originates from the seminal studies by \citep{sohl2015deep} and \citep{hoogeboom2021argmax}. This paradigm was significantly extended by the Structured Denoising Diffusion Probabilistic Models (D3PM) framework \citep{austin2021structured}, which employs a discrete state Markov chain to systematically inject noise into the input sequence during the forward phase. While this initial formulation relied on discrete steps, subsequent research investigated the methodology within continuous-time settings \citep{campbell2022continuous}. Concurrently, SEDD \citep{lou2023discrete} introduced an alternative strategy by directly calculating likelihood ratios and incorporating a novel denoising score entropy objective. Recent theoretical and empirical investigations—exemplified by MDLM \citep{shi2024simplified,sahoo2024simple,zheng2024masked} and RADD \citep{ou2024your}—have demonstrated a crucial equivalence: various distinct parameterizations of discrete diffusion models are mathematically identical. This finding has been central to unifying and simplifying the understanding of these architectures.

Inspired by these algorithmic and theoretical advancements, the scale of diffusion models has expanded significantly, recently reaching the 7--8 billion parameter regime. For instance, LLaDA \citep{nie2025large} was trained from scratch using a weighted cross-entropy loss, whereas Dream \citep{ye2025dream} was adapted from the established Qwen2.5 base model, achieving competitive performance with substantially reduced data requirements. Prominent commercial implementations, such as Mercury \citep{labs2025mercuryultrafastlanguagemodels}, Gemini Diffusion \citep{gemini-diffusion}, and Seed Diffusion \citep{song2025seeddiffusionlargescalediffusion}, further illustrate the maturity of this field. A notable characteristic of these models is that they achieve external benchmark performance comparable to larger autoregressive (AR) language models while exhibiting superior decoding efficiency. Collectively, these results affirm the potential of diffusion language models as a compelling alternative within the generative AI landscape.

\subsection{Inference Optimization for Diffusion Language Models}
Alongside architectural advancements, researchers are actively optimizing the inference phase to achieve a better balance of latency and performance. This work largely centers on two approaches: implementing efficient caching protocols and developing advanced decoding algorithms. To mitigate the substantial computational cost imposed by the bidirectional attention mechanism, which represents a significant bottleneck in DLMs inference, researchers have proposed several caching techniques~\citep{wu2025fastdllmtrainingfreeaccelerationdiffusion, ma2025dkv, liu2025dllm, song2025sparse}, which focus on how and when the Key-Value (KV) activations are stored or retrieved. Some other technique modifies the attention mechanism itself to reduce computational load ~\citep{chen2025dpad}.

Another direction focuses on optimizing decoding strategies to enhance convergence and efficiency. One primary direction involves dynamically adjusting the unmasking process based on metric value assigned to each token prediction, such as confidence. For example, Dimple \citep{yu2025dimple} employs confident decoding, which adjusts the generation budget (e.g., the number of tokens unmasked) at each step according to a confidence threshold. Similarly, SlowFast Sampling \citep{wei2025acceleratingdiffusionlargelanguage} proposes a dynamic strategy that adaptively alternates between exploratory and acceleration steps. Prophet \citep{li2025diffusion} further refines this by dynamically deciding whether to enable early-commit decoding based on the confidence gap between the top-k predicted tokens. WINO \citep{hong2025wide} introduces a parallel draft-and-verify scheme that accelerates decoding by aggressively drafting tokens and concurrently refining those with low confidence. To improve generation quality, \citep{yang2025tamingmaskeddiffusionlanguage} introduced end-of-sequence (EOS) token rejection during early decoding stages combined with an ascending decoding budget. Furthermore, \citep{li2025diffuspecunlockingdiffusionlanguage} incorporated speculative decoding into Diffusion Language Models (DLMs) by utilizing a pre-trained DLM to produce multi-token drafts in a single forward pass.

A distinct category of optimization leverages information from historical decoding steps to improve consistency. \citep{wang2025time} introduced temporal self-consistency voting, which selects the most reliable output by identifying the most consistent sequence across all decoding steps. Additionally, dInfer \citep{ma2025dinferefficientinferenceframework} proposes credit decoding, where the predictive distribution is updated by fusing logit scores with a credit score that explicitly quantifies token consistency. 

Instead of directly altering the sampling distribution through logit fusion, our method adopts a novel perspective by modeling the consistency of historical steps using the conditional mutual information between the context and token predictions. This offers a more theoretically grounded and fine-grained measure of consistency. By explicitly minimizing this term, our approach naturally promotes coherent and error-resistant sequences without relying on heuristic score adjustments, resulting in superior robustness for complex generation tasks.
\section{Conclusion}
\label{sec:conclusion}
In this work, we presented Coherent Contextual Decoding (CCD), a framework that fundamentally reshapes Diffusion Language Model sampling by leveraging historical contexts. Our approach utilizes a trajectory rectification mechanism—grounded in conditional mutual information—to reject sub-optimal paths without compromising coherence. Notably, our adaptive strategy, CCD-DS, breaks the conventional trade-off between speed and accuracy. By replacing uniform decoding budgets with adaptive sampling based on consistency, it realizes a synchronous enhancement of efficiency and quality, delivering up to $3.48\times$ inference speedup and $3.91\%$ performance improvement across diverse benchmarks on both Dream and LLaDA.

%Bibliography
\bibliographystyle{unsrt}  
\bibliography{references}  

\clearpage

\begin{appendices}

\section{Appendix}
\label{sec:Appendix}

\begin{table}[!h]
  \renewcommand{\arraystretch}{0.5}
\centering
\begin{tabular}{c}
\begin{tcolorbox}[title={\textbf{Task:} \emph{Lorie earns \$10 per hour. Karen earns twice what Lorie earns. How much does Karen earn in two days if she works 3 hours per day?
\quad \colorbox{red!30}{\texttt{Incorrect}}\quad \colorbox{green!30}{\texttt{Correct}}}\quad \colorbox{blue!10}{$\mathbf{\otimes}$} Mask token}, colback=white,colframe=black!10, coltitle=black, width=\textwidth]
\noindent \textbf{Baseline:}\; \texttt{Kylar wants to buy 16 glasses, so he needs to pay for 16 glasses in total.}\\
\texttt{The first glass costs \$5}\textcolor{red}{\underline{.}}\tikz[baseline=-0.5ex]{\draw[red, dashed, ->, thick] (0,0) -- (0.8,0);} \textcolor{red}{\textit{The generative trajectory starts to be different from our results at the "." }}\\
\texttt{The second glass costs 60\% of the price, which is \$5 *0.6 = \$3.}\\
\texttt{The third glass costs \$5.}\\
\texttt{The fourth glass costs 60\% of the price,whichis \$5*0.6 =\$3The fifth glass costs \$5.The sixth glass costs 60\%of the price,which is \$5 *0.6 = \$3.}\\
\texttt{The seventh glass costs \$5.}\\
\texttt{The eighth glass costs60\% of the price,which is \$5* 0.6 = \$3.}\\
\texttt{The ninth glass costs \$5.}\\
\texttt{The tenth glass costs60\% of the price,which is \$5 *0.6 = \$3.}\\
\texttt{The eleventh glass costs \$5.}\\
\texttt{The twelfth glass costs 60\% of the price, which is \$5 * 0.6 = \$3.The thirteenth glass costs \$5.The fourteenth glass costs 60\% of the price, which is \$5 * 0.6 = \$3.The fifteenth glass costs \$5.The sixteenth glass costs 60\% of the price, which is \$5 * 0.6 = \$3.Kylar needs to pay\$5*8+\$3*8=\$40+\$24=\$64 for the glasses.}\\
\texttt{The answer is:\colorbox{green!30}{64}}\\
\noindent\hdashrule{\textwidth}{0.5pt}{2pt}\\
\vspace{0.3em}
\noindent\colorbox{gray!20}{\parbox{\dimexpr\textwidth-2\fboxsep}{\centering\vspace{0.2em}\textbf{Illustration of diffusion intermediate process  when the generative trajectroty starts to be separated}\vspace{0.2em}}}
\vspace{-0.8em}\\
\noindent\textcolor{gray}{\scriptsize token index \texttt{\hspace{5.2em}20\hspace{2.5em}21\hspace{3.3em}22\hspace{2.9em}23\hspace{3.em}24\hspace{0.8em}25\hspace{1.em}26\hspace{1.em}27\hspace{1.2em}28\hspace{1.em}29\hspace{1.4em}}}\\
\noindent \textbf{Baseline:}\; \texttt{... The first glass costs \$5} \colorbox{blue!10}{$\mathbf{\otimes}$} \colorbox{blue!10}{$\mathbf{\otimes}$} \colorbox{blue!10}{$\mathbf{\otimes}$} \colorbox{blue!10}{$\mathbf{\otimes}$} \colorbox{blue!10}{$\mathbf{\otimes}$} \textbf{...} \tikz[baseline=-0.5ex]{\draw[red, dashed, ->, thick] (0,0) -- (0.7,0);} \textcolor{red}{\textit{Step \#25}}
\begin{center}
\vspace{-0.1cm}
\quad\quad\quad\quad\quad\quad\quad\quad\quad\quad\quad\quad\quad
\tikz{\draw[->, thick, line width=1.pt] (0,0) -- (0,-0.5) node[right, pos=0.5] {\textit{diffusion, do sampling procedure in Eq.(\ref{eq:sampling procedure equation 1})}};}
\end{center}
\colorbox{gray!20}{\parbox{\dimexpr\textwidth-2\fboxsep}{\texttt{Top-1 confident index of single-step predictive distributions: \textcolor{red}{25th mask token}}}}
\begin{center}
\vspace{-0.1cm}
\tikz{\draw[->, thick, line width=1.pt] (0,0) -- (0,-0.5) node[right, pos=0.5] {\textit{decoding}};}
\end{center}
\vspace{-0.1cm}
\texttt{... The first glass costs \$5.} \colorbox{blue!10}{$\mathbf{\otimes}$} \colorbox{blue!10}{$\mathbf{\otimes}$} \colorbox{blue!10}{$\mathbf{\otimes}$}
\colorbox{blue!10}{$\mathbf{\otimes}$} 
\textbf{...} \tikz[baseline=-0.5ex]{\draw[red, dashed, ->, thick] (0,0) -- (1,0);} \textcolor{red}{\textit{Step \#25}}
\end{tcolorbox}
\\[0.1em]
\begin{tcolorbox}[title={}, colback=white,colframe=blue!15, coltitle=black, width=\textwidth]
\noindent \textbf{Ours:}\;  \texttt{Kylar wants to buy 16 glasses, so he needs to pay for 16 glasses in total.}\\
\texttt{The first glass costs \$5, and every second glass costs 60\% of the price, which is \$5 * 60/100 = \$3.
So, every second glass costs \$3.}\\
\texttt{Since Kylar wants to buy 16 glasses, he will have to pay for 8 glasses at the regular price of \$5 and 8 glasses at the discounted price of \$3.}\\
\texttt{The cost of 8 glasses at the regular price is 8 * \$5 = \$40.}\\
\texttt{The cost of 8 glasses at the discounted price is 8 * \$3 = \$24.}\\
\texttt{Therefore, Kylar needs to pay \$40 + \$24 = \$64 for the 16 glasses.}\\
\texttt{The answer is:\colorbox{green!30}{64}}\\
\noindent\hdashrule{\textwidth}{0.5pt}{2pt}\\
\vspace{0.3em}
\noindent\colorbox{gray!20}{\parbox{\dimexpr\textwidth-2\fboxsep}{\centering\vspace{0.2em}\textbf{Illustration of diffusion intermediate process when the generative trajectroty starts to be separated}\vspace{0.2em}}}
\vspace{-0.8em}\\
\noindent\textcolor{gray}{\scriptsize token index \texttt{\hspace{5.2em}20\hspace{2.5em}21\hspace{3.3em}22\hspace{2.9em}23\hspace{3.em}24\hspace{0.8em}25\hspace{1.em}26\hspace{1.em}27\hspace{1.2em}28\hspace{1.em}29\hspace{1.4em}}}\\
\noindent \textbf{Baseline:}\; \texttt{... The first glass costs \$5} \colorbox{blue!10}{$\mathbf{\otimes}$} \colorbox{blue!10}{$\mathbf{\otimes}$} \colorbox{blue!10}{$\mathbf{\otimes}$} \colorbox{blue!10}{$\mathbf{\otimes}$} \colorbox{blue!10}{$\mathbf{\otimes}$} \textbf{...} \tikz[baseline=-0.5ex]{\draw[red, dashed, ->, thick] (0,0) -- (0.7,0);} \textcolor{red}{\textit{Step \#25}}
\begin{center}
\vspace{-0.1cm}
\quad\quad\quad\quad\quad\quad\quad\quad\quad\quad\quad\quad\quad\quad
\tikz{\draw[->, thick, line width=1.pt] (0,0) -- (0,-0.5) node[right, pos=0.5] {\textit{diffusion, do sampling procedure in Eq.(\ref{pratical decoding})}};}
\colorbox{gray!20}{\parbox{\dimexpr\textwidth-2\fboxsep}{\texttt{Top-1 confident index of approximated target distributions: \textcolor{red}{28th mask token}}}}

\end{center}
\begin{center}
\vspace{-0.2cm}
\tikz{\draw[->, thick, line width=1.pt] (0,0) -- (0,-0.5) node[right, pos=0.5] {\textit{decoding}};}
\end{center}
\vspace{-0.2cm}
\texttt{... The first glass costs \$5} 
\colorbox{blue!10}{$\mathbf{\otimes}$} 
\colorbox{blue!10}{$\mathbf{\otimes}$} \colorbox{blue!10}{$\mathbf{\otimes}$}
second
\colorbox{blue!10}{$\mathbf{\otimes}$} \textbf{...} \tikz[baseline=-0.5ex]{\draw[red, dashed, ->, thick] (0,0) -- (0.3,0);} \textcolor{red}{\textit{Reject single-step top-1 confident index}}
\end{tcolorbox}
\end{tabular}
\caption{Final outputs of different sampling procedures on an example of the GSK8K benchmark. We present the reason analysis of the different final outputs by tracking the trajectory of the separated points.  }
\label{gsm8k2}
\end{table}

\begin{table}[!h]
  \renewcommand{\arraystretch}{0.8}
\centering
\begin{tabular}{c}
\begin{tcolorbox}[title={\textbf{Task:} \emph{You plan to visit 4 European cities for 22 days in total. You only take direct flights to commute between cities. You would like to visit Warsaw for 6 days. You would like to meet your friends at Warsaw between day 1 and day 6 to tour together. You would like to visit Vienna for 6 days. You want to spend 6 days in Mykonos. You want to spend 7 days in Vilnius. Here are the cities that have direct flights: Warsaw and Vilnius, Warsaw and Vienna, Vilnius and Vienna, Vienna and Mykonos. Find a trip plan of visiting the cities for 22 days by taking direct flights to commute between them. SOLUTION:\quad \colorbox{red!30}{\texttt{Incorrect}}\quad \colorbox{green!30}{\texttt{Correct}}}\quad \colorbox{blue!10}{$\mathbf{\otimes}$} Mask token}, colback=white,colframe=black!10, coltitle=black, width=\textwidth]
\noindent \textbf{Baseline:}\; \texttt{Here is the trip plan for visiting the 4 European cities for 22 days:}\\
\texttt{**Day 1-6:** Arriving in Warsaw and visit Warsaw for 6 days.}\\
\texttt{**Day 6:** Fly from Warsaw to Vilnius.}\\
\texttt{**Day 6-12:** Visit Vilnius for 7 days.}\\
\texttt{**Day 12:** Fly from Vilnius to Vienna.}\\
\texttt{**Day 12-1}\colorbox{red!30}{\texttt{8}}\texttt{:** Visit Vienna for 6 days.} \tikz[baseline=-0.5ex]{\draw[red, dashed, ->, thick] (0,0) -- (0.9,0);} \textcolor{red}{\textit{It should be 7 instead of 8 to match the 6 days}}\\
\texttt{**Day 1}\colorbox{red!30}{\texttt{8}}\texttt{:** Fly from Vienna to Mykonos.}\tikz[baseline=-0.5ex]{\draw[red, dashed, ->, thick] (0,0) -- (2,0);} \textcolor{red}{\textit{Accumulated error from the last line}}\\
\texttt{**Day 1}\colorbox{red!30}{\texttt{8}}\texttt{-22:** Visit Mykonos for 6 days.}\tikz[baseline=-0.5ex]{\draw[red, dashed, ->, thick] (0,0) -- (2,0);} \textcolor{red}{\textit{Accumulated error from the last line}}\\
\noindent\hdashrule{\textwidth}{0.5pt}{2pt}\\
\vspace{0.3em}
\noindent\colorbox{gray!20}{\parbox{\dimexpr\textwidth-2\fboxsep}{\centering\vspace{0.2em}\textbf{Illustration of diffusion intermediate process when three mask tokens left}\vspace{0.2em}}}
\vspace{-0.8em}\\
\noindent \textbf{Baseline:}\; \texttt{Here is the trip plan for visiting the 4 European cities for 22 days:}\\
\texttt{**Day 1-6:** Arriving in Warsaw and visit Warsaw for 6 days.}\\
\texttt{**Day 6:** Fly from Warsaw to Vilnius.}\\
\texttt{**Day 6-12:** Visit Vilnius for 7 days.}\\
\texttt{**Day 12:** Fly from Vilnius to Vienna.}\\
\texttt{**Day 12-1}\colorbox{red!30}{\texttt{8}}\texttt{:** Visit Vienna for 6 days.} \tikz[baseline=-0.5ex]{\draw[red, dashed, ->, thick] (0,0) -- (0.3,0);} \textcolor{red}{\textit{Semantically ambiguous token was decoded early.}}\\
\texttt{**Day 1}\colorbox{red!30}{\texttt{8}}\texttt{:** Fly from Vienna to Mykonos.}\\
\texttt{**Day 1}\colorbox{red!30}{\texttt{8}}\texttt{-22:** Visit Mykonos for} \colorbox{blue!10}{$\mathbf{\otimes}$} \colorbox{blue!10}{$\mathbf{\otimes}$} \colorbox{blue!10}{$\mathbf{\otimes}$}\tikz[baseline=-0.5ex]{\draw[red, dashed, ->, thick] (0,0) -- (0.9,0);} \textcolor{red}{\textit{Left last three mask tokens at the end of sentence}}
\end{tcolorbox}\\
[0.1em]
\begin{tcolorbox}[title={}, colback=white,colframe=blue!15, coltitle=black, width=\textwidth]

\noindent \textbf{Ours:}\; \texttt{Here is the trip plan for visiting the 4 European cities for 22 days:}\\
\texttt{**Day 1-6:** Arriving in Warsaw and visit Warsaw for 6 days.}\\
\texttt{**Day 6:** Fly from Warsaw to Vilnius.}\\
\texttt{**Day 6-12:** Visit Vilnius for 7 days.}\\
\texttt{**Day 12:** Fly from Vilnius to Vienna.}\\
\texttt{**Day 12-1}\colorbox{green!30}{\texttt{7}}\texttt{:** Visit Vienna for 6 days.}\tikz[baseline=-0.5ex]{\draw[red, dashed, ->, thick] (0,0) -- (2,0);} \textcolor{red}{\textit{Correct}}\\
\texttt{**Day 1}\colorbox{green!30}{\texttt{7}}\texttt{:** Fly from Vienna to Mykonos.}\tikz[baseline=-0.5ex]{\draw[red, dashed, ->, thick] (0,0) -- (2,0);} \textcolor{red}{\textit{Correct}}\\
\texttt{**Day 1}\colorbox{green!30}{\texttt{7}}\texttt{-22:** Visit Mykonos for 6 days.}\tikz[baseline=-0.5ex]{\draw[red, dashed, ->, thick] (0,0) -- (2,0);} \textcolor{red}{\textit{Correct}}\\
\noindent\hdashrule{\textwidth}{0.5pt}{2pt}\\
\vspace{-0.8em}
\noindent\colorbox{gray!20}{\parbox{\dimexpr\textwidth-2\fboxsep}{\centering\vspace{0.2em}\textbf{Illustration of diffusion intermediate process when three mask tokens left}\vspace{0.2em}}}
\vspace{0.3em}\\
\noindent \textbf{Ours:}\; \texttt{Here is the trip plan for visiting the 4 European cities for 22 days:}\\
\texttt{**Day 1-6:** Arriving in Warsaw and visit Warsaw for 6 days.}\\
\texttt{**Day 6:** Fly from Warsaw to Vilnius.}\\
\texttt{**Day 6-12:** Visit Vilnius for 7 days.}\\
\texttt{**Day 12:** Fly from Vilnius to Vienna.}\\
\texttt{**Day 12-1}  \colorbox{blue!10}{$\mathbf{\otimes}$} \texttt{:** Visit Vienna for 6 days.}\tikz[baseline=-0.5ex]{\draw[red, dashed, ->, thick] (0,0) -- (0.4,0);} \textcolor{red}{\textit{Semantically ambiguous token maintains masked.}}\\
\texttt{**Day 1} \colorbox{blue!10}{$\mathbf{\otimes}$} \texttt{:** Fly from Vienna to Mykonos.}\tikz[baseline=-0.5ex]{\draw[red, dashed, ->, thick] (0,0) -- (0.4,0);} \textcolor{red}{\textit{Semantically ambiguous token maintains masked.}}\\
\texttt{**Day 1} \colorbox{blue!10}{$\mathbf{\otimes}$} \texttt{-22:** Visit Mykonos for 6 days.}\tikz[baseline=-0.5ex]{\draw[red, dashed, ->, thick] (0,0) -- (0.4,0);} \textcolor{red}{\textit{Semantically ambiguous token maintains masked.}}
\end{tcolorbox}
\end{tabular}
\caption{Upper: final outputs of different sampling procedures on an example of the Trip benchmark. Bottom: The reason analysis of the different final outputs by tracking the late-stage generative results (three mask tokens are left). }
\label{tab:trip example}
\vspace{-0.5cm}
\end{table}

\textbf{-- Our proposed method can encourage more reasonable generative trajectories.} As demonstrated in Table~\ref{gsm8k2}, our method leverages global context to make superior decoding decisions at critical divergence points. At Step \#25 after "\texttt{The first glass costs \$5}", the baseline selects the period "\texttt{.}" (25th mask token) based on single-step predictive confidence.  In contrast, our approximated target distributions identify the 28th mask token "\texttt{second}" as optimal, enabling "\texttt{, and every second glass costs...}". This choice leads to a superior reasoning strategy: establishing the general pricing pattern first, then performing a consolidated calculation ("\texttt{8 glasses at \$5 and 8 glasses at \$3}"). 

This exemplifies a key mechanistic advantage: while single-step confidence gravitates toward syntactically safe choices (sentence-ending punctuation), our cross-iteration consistency recognizes that continuing with "\texttt{second}" enables more elegant problem decomposition. Trajectory-level optimization fundamentally improves solution quality in reasoning tasks where strategic planning outweighs local fluency.

\textbf{-- Our proposed method can automatically defer the decoding of semantically ambiguous tokens.}
As shown in Table \ref{tab:trip example} , when three mask tokens remain, the baseline has already incorrectly decoded "8" (should be "7") at position "Day 12-1\_". This premature decoding occurs because the baseline's uniform sampling budget forces token selection based on single-step predictive distribution alone. In contrast, our method identifies this token as having high context sensitivity and shows inconsistent predictions across iterations in the historical buffer. Therefore, our proposed method automatically defers its decoding by maintaining it as masked.

This deferral is crucial for arithmetic-dependent reasoning tasks. By keeping these ambiguous tokens masked longer, our method can leverage the constraint that total days must sum to 22, which becomes clearer as surrounding context solidifies. The baseline's early mistake cascades through subsequent tokens ("Day 18" instead of "Day 17"), while our sampling procedure naturally preserves tokens with low cross-iteration agreement for later stages when sufficient context is available, preventing error accumulation in structured reasoning where tokens have strong interdependencies.  

\end{appendices}

\end{document}